\documentclass[journal]{IEEEtran}
\usepackage{amsmath,amsfonts,amsthm}
\usepackage{algorithmic}
\usepackage{algorithm}
\usepackage{array}
\usepackage[caption=false,font=normalsize,labelfont=sf,textfont=sf]{subfig}
\usepackage{textcomp}
\usepackage{stfloats}
\usepackage{url}
\usepackage{verbatim}
\usepackage{graphicx}

\usepackage{dsfont}
\usepackage{colortbl}

\definecolor{customcellcolor}{rgb}{1,0.8,0.8}
\usepackage[all,cmtip]{xy}
\usepackage{mathtools}

\DeclarePairedDelimiter\floor{\lfloor}{\rfloor}
\newtheorem{theorem}{Theorem}
\newtheorem{lemma}[theorem]{Lemma}
\newtheorem{proposition}[theorem]{Proposition}
\newtheorem{definition}[theorem]{Definition}
\newtheorem{remark}[theorem]{Remark}

\usepackage[backend=biber, sorting=none]{biblatex}
\bibliography{main}

\usepackage{scalerel}
\usepackage{tikz}
\usetikzlibrary{svg.path}

\definecolor{orcidlogocol}{HTML}{A6CE39}
\tikzset{
  orcidlogo/.pic={
    \fill[orcidlogocol] svg{M256,128c0,70.7-57.3,128-128,128C57.3,256,0,198.7,0,128C0,57.3,57.3,0,128,0C198.7,0,256,57.3,256,128z};
    \fill[white] svg{M86.3,186.2H70.9V79.1h15.4v48.4V186.2z}
                 svg{M108.9,79.1h41.6c39.6,0,57,28.3,57,53.6c0,27.5-21.5,53.6-56.8,53.6h-41.8V79.1z M124.3,172.4h24.5c34.9,0,42.9-26.5,42.9-39.7c0-21.5-13.7-39.7-43.7-39.7h-23.7V172.4z}
                 svg{M88.7,56.8c0,5.5-4.5,10.1-10.1,10.1c-5.6,0-10.1-4.6-10.1-10.1c0-5.6,4.5-10.1,10.1-10.1C84.2,46.7,88.7,51.3,88.7,56.8z};
  }
}

\newcommand\orcidicon[1]{\href{https://orcid.org/#1}{\mbox{\scalerel*{
\begin{tikzpicture}[yscale=-1,transform shape]
\pic{orcidlogo};
\end{tikzpicture}
}{|}}}}

\usepackage[hidelinks]{hyperref}

\begin{document}

\title{Machine Learning Predictors \\ for Min-Entropy Estimation}

\author{Javier Blanco-Romero$^{\orcidicon{0009-0004-0635-953X}}$,
        Vicente Lorenzo$^{\orcidicon{0000-0003-2077-6095}}$,
        Florina Almenares Mendoza$^{\orcidicon{0000-0002-5232-2031}}$,
        Daniel Díaz-Sánchez$^{\orcidicon{0000-0002-3323-6453}}$
\thanks{Javier Blanco-Romero is with the Department of Telematic Engineering, Universidad Carlos III de Madrid, Leganés, Madrid, 28911, Spain (e-mail: frblanco@pa.uc3m.es).}
\thanks{Vicente Lorenzo is with the Department of Telematic Engineering, Universidad Carlos III de Madrid, Leganés, Madrid, 28911, Spain (e-mail: vilorenz@pa.uc3m.es).}
\thanks{Florina Almenares Mendoza is with the Department of Telematic Engineering, Universidad Carlos III de Madrid, Leganés, Madrid, 28911, Spain (e-mail: florina@it.uc3m.es).}
\thanks{Daniel Díaz-Sánchez is with the Department of Telematic Engineering, Universidad Carlos III de Madrid, Leganés, Madrid, 28911, Spain (e-mail: dds@it.uc3m.es).}
}

\markboth{Journal of \LaTeX\ Class Files,~Vol.~14, No.~8, August~2021}%
{Shell \MakeLowercase{\textit{et al.}}: A Sample Article Using IEEEtran.cls for IEEE Journals}


\maketitle

\begin{abstract}
This study investigates the application of machine learning predictors for min-entropy estimation in Random Number Generators (RNGs), a key component in cryptographic applications where accurate entropy assessment is essential for cybersecurity. Our research indicates that these predictors, and indeed any predictor that leverages sequence correlations, primarily estimate average min-entropy, a metric not extensively studied in this context. We explore the relationship between average min-entropy and the traditional min-entropy, focusing on their dependence on the number of target bits being predicted. Utilizing data from Generalized Binary Autoregressive Models, a subset of Markov processes, we demonstrate that machine learning models (including a hybrid of convolutional and recurrent Long Short-Term Memory layers and the transformer-based GPT-2 model) outperform traditional NIST SP 800-90B predictors in certain scenarios. Our findings underscore the importance of considering the number of target bits in min-entropy assessment for RNGs and highlight the potential of machine learning approaches in enhancing entropy estimation techniques for improved cryptographic security.
\end{abstract}

\begin{IEEEkeywords}
Min-entropy Estimation, Machine Learning Predictors, Random Number Generators, Autoregressive Processes, Generalized Binary Autoregressive Models
\end{IEEEkeywords}

\section{Introduction}
\IEEEPARstart{T}{he} 
 security of cryptographic systems often hinges on the generation of random values. Although there is a broad spectrum of algorithms and devices used to generate these random values, they are all generically denoted by Random Number Generators (RNGs). Given the important role that RNGs play in the context of cybersecurity,
it becomes evident that rigorous criteria are necessary for evaluating the reliability and performance of an RNG.

Multiple approaches are commonly employed to assess the quality of the output of an RNG (cf. \cite{bassham2010sp}, \cite{marsaglia2008marsaglia}, \cite{abbott2019experimentally}, \cite{calude2010experimental}, \cite{kavulich2021searching}, \cite{bird2020effects}, etc.).
In this paper the emphasis will be put on:
\begin{itemize}
    \item \textbf{Entropy} tests, as those found in NIST Special Publication 800-90B \cite{turan2018recommendation}, which estimate the entropy of a noise source based on appropriate samples (cf. \cite{abraham2022high}, \cite{islam2022using}, \cite{li2020jitter}, etc.).

    \item \textbf{Machine Learning} models trained with the output of an RNG aiming to guess the bit or set of bits that  follow a given sequence, which can give an insight into how predictable the output of the RNG is
    (cf. \cite{truong2018machine}, \cite{yang2018neural}, \cite{lv2020high}, \cite{li2020deep}, \cite{feng2020testing}, \cite{li2023improvement}, etc.).
\end{itemize}

The fact that the entropy of a given source and the predictability of its output are correlated was already noticed by  Shannon \cite{shannon1951prediction}. Nevertheless, the link between these two concepts is far from being completely understood, specially if one takes into account the heterogeneity of entropy definitions that can be found in the literature and how much the predictability of the output of an entropy source relies on the predictor being considered. Building on the evidence provided by \cite{kelsey2015predictive} that the entropy estimators considered by NIST Special Publication 800-90B \cite{turan2018recommendation} tend to underestimate min-entropy, an attempt to reinforce the argument \cite{zhu2017analysis} that predictors are better suited to estimate average min-entropy \cite{dodis2004fuzzy} than min-entropy is carried out in this paper. In the first stage, the theoretical framework required to support our thesis is developed. In the second stage, experimental validation of the theoretical analysis is conducted.

Whereas \cite{kelsey2015predictive} concentrates on  Ensemble, Categorical Data, and Numerical predictors,  the focus of this paper will be on machine learning predictors. In particular, 
a hybrid model that integrates convolutional and recurrent Long Short-Term Memory (LSTM) layers and the transformer-based GPT-2 model will be considered. As in \cite{zhu2017analysis}, we generate sets of data for which a theoretical entropy can be calculated so that the machine learning entropy estimation can be compared to the theoretical value. Nevertheless, while the data generated in \cite{zhu2017analysis} comes from an oscillator-based model and Markov processes of order at most $2$, our data comes from Generalized Binary Autoregressive Models \cite{jentsch2019generalized}, a subclass of Markov chains  that allows us to easily parameterize correlations and anticorrelations at the bit level and compute min-entropies.

Our research also investigates the influence of the number of target bits on the estimation of min-entropy. We demonstrate that the relationship between average min-entropy and min-entropy is significantly affected by the number of target bits being predicted. This finding highlights the importance of considering the target bit count when assessing the min-entropy of RNGs using machine learning predictors.

The remainder of this paper is structured as follows: Section 2 presents a literature review, discussing the current state of the art in the application of predictors for min-entropy estimation. In Section 3, we establish the theoretical framework, where we study the concept of average min-entropy and its relationship with min-entropy, deriving a series of results for order-p Markov chains and gbAR(p). Section 4 outlines our experimental methodology, aimed at validating the theoretical findings. Section 5 presents the results of our experiments, followed by Section 6, which offers a discussion of these results and their implications. Finally, Section 7 concludes the paper with a summary of our findings and suggestions for future research.

\subsection{Notation and Conventions}
The following notation and conventions will be considered throughout this paper:
\begin{itemize}
    \item Random Variables: Uppercase letters \(X_1, X_2, A, \ldots\) represent random variables, while their corresponding realizations are represented by lowercase letters \(x_1, x_2, a, \ldots\) and by abusing the notation $P(A=a)$ will be denoted by $P(a)$. Furthermore, by $X\sim B(n,p)$ we mean that $X$ is a random variable that follows a binomial distribution with  number of trials $n$ and a success probability $p$, and by $(X_1,\ldots,X_k)\sim \text{Mult}(n;p_1,\ldots,p_k)$ we mean that $(X_1,\ldots,X_k)$ is a multivariate random variable that follows a multinomial distribution with number of trials $n$ and probability vector $(p_1,\ldots,p_k)$.
    \item Expected Values: The notation \(\langle \cdot \rangle\) is used to indicate expected values. Given discrete random variables $X_{t-p},\ldots, X_{t+n}$ where $t,p,n\in\mathbb{Z}$ and $p,n$ are non-negative, we will be particularly interested in the following type of expression:
    \begin{equation*}
      \resizebox{0.9\hsize}{!}{$
        \begin{aligned}
            \left\langle \max_{x_{t},\ldots,x_{t+n}} P( x_{t},\ldots,x_{t+n} \mid x_{t-1},\ldots,x_{t-p}) \right\rangle_{x_{t-1},\ldots,x_{t-p}}    =
        \\ \sum_{x_{t-1},\ldots, x_{t-p}}P(x_{t-1},\ldots, x_{t-p})\max_{x_t,\ldots,x_{t+n}}P(x_{t},\ldots, x_{t+n}\mid x_{t-1},\ldots, x_{t-p}).
        \end{aligned}
        $}
    \end{equation*}

    \item Logarithms: All logarithmic functions are considered to be base 2 and are denoted by \(\log\).
\end{itemize}

\section{State of the Art (Literature Review)}

\subsection{Entropies}

The relationship between entropy and the predictability of a sequence was first investigated by Shannon \cite{shannon1951prediction}, who noticed that 
the problem of prediction is fundamentally connected to the concept of entropy. Min-entropy, denoted as \(H_\infty(X)\), represents the negative logarithm of the probability of a correct guess on the random variable \(X\) under an optimal strategy \cite{cachin1997entropy}. Mathematically, the probability of guessing the most likely output of an entropy source can be expressed as:

\begin{equation*}
\label{min-entropy}
2^{-H_\infty(X)} = \max_{x \in X} P_X(x) \ .
\end{equation*}

In cryptography, min-entropy is an important measure, as it provides a conservative estimate of the difficulty of guessing or predicting the most likely output of the entropy source, as emphasized in the NIST Recommendation \cite{turan2018recommendation}.

Moreover, entropy estimation is complex when the output distribution is unknown, and typical assumptions like outputs being independent and identically distributed (i.i.d.) do not apply. Good entropy estimation needs understanding of the underlying nondeterministic process of the entropy source, and statistical tests, as referenced, can only act as a sanity check on such an estimate \cite{kelsey2015predictive}.

In this context, the concept of predictors has been introduced. As described by Kelsey et al., a predictor contains a dynamic model that operates through a four-step process: 1) Assume a probability model for the source, 2) Estimate the model's parameters from the input sequence on-the-fly, 3) Use these parameters to attempt to predict the still-unseen values in the input sequence, and 4) Estimate the min-entropy of the source from the performance of these predictions \cite{kelsey2015predictive}. Unlike traditional machine learning methods, this approach is parametric and relies on a model of the underlying probability distribution. Another difference from traditional supervised learning methods, which separate training and testing sets, is that predictors remain in the training phase indefinitely, allowing for continuous adaptation and improvement in prediction accuracy. Predictors are characterized by two primary performance metrics. The first, global predictability, gauges the long-term accuracy of predictions. Specifically, a predictor's global accuracy \( p_{\text{acc}} \) represents the probability that it will correctly predict a given sample from a noise source over an extended sequence, effectively measuring the percentage of correct predictions. The second, local predictability, emphasizes the length of the longest streak of correct predictions, becoming important when the source produces highly predictable outputs in short spurts. The final entropy estimate for a predictor is determined by the lesser value between the global and local entropy estimates, represented by  $\hat{H} = \min(\hat{H}_{\text{global}}, \hat{H}_{\text{local}})$ . 

Hence, predictors play a significant role in setting bounds on an attacker's performance, linking predictability to min-entropy. For a description of the evolution of the introduction of predictors on the NIST SP 800-90B see \cite{lv2020high}.

Zhu et al. examined the issue of underestimation in non-IID data pertaining to the NIST collision and compression test, proposing an enhanced method to address the underestimation of min-entropy \cite{zhu2017analysis}. They introduced a novel formula specifically aimed at the high-order Markov process, founded on the principles of conditional probability. Furthermore, they highlighted that the correct prediction probability within a predictor can also be understood as a form of conditional probability. 

Zhu's min-entropy formula for the Markov process can be related to the concept of average min-entropy as defined by Dodis \cite{dodis2004fuzzy}. Average min-entropy considers the predictability of a random variable given another possibly correlated random variable and can be expressed as:

\begin{equation*}
\label{average_conditional_min_entropy}
\begin{aligned}
    \tilde{H}_\infty(A \mid B) = - \log \left( \langle \max_{a} P( a \mid b) \rangle_{b} \right) = \\ - \log \left( \langle 2^{-H_\infty(A\mid B=b)}\rangle_{b} \right) \ ,
\end{aligned}
\end{equation*}

where $$H_\infty(A\mid B=b)=P(b)\cdot \max_a P(a\mid b).$$ Dodis' definition on average min-entropy offers valuable insights into the logarithm of predictability, presenting it as a “worst-case” entropy measure \cite{dodis2004fuzzy}. \\

Several works have considered the problem of designing machine learning predictors for the evaluating RNGs. Here, we examine some of the most relevant contributions in relation to min-entropy estimation.

Truong et al. \cite{truong2018machine} introduced the use of a recurrent convolutional neural network (RCNN) to analyze quantum random number generators (QRNGs). This RCNN model was employed to evaluate different stages of an optical continuous variable QRNG. The study focused on detecting inherent correlations, particularly under the influence of deterministic classical noise. Their methodology included a comprehensive analysis, from examining the robustness of QRNGs against machine learning attacks to benchmarking with a congruential pseudo-random number generator (CRNG). Their model's prediction accuracy was compared with the guessing probability of the data distribution, effectively entailing a comparison with the min-entropy.

Yang et al. and Lv et al. explored neural network-based min-entropy estimation for random number generators \cite{yang2018neural}, \cite{lv2020high}. Their approach involved training predictive models on simulated data, where the correct entropy was ascertainable due to the known output distributions. Additionally, their study included a performance analysis and comparison of their results with the NIST SP 800-90B's predictors, providing a detailed examination of the efficacy and accuracy of their neural network-based approach in entropy estimation.

Li et al. \cite{li2020deep} proposed a deep learning-based predictive analysis to assess the security of a non-deterministic random number generator (NRNG) using white chaos. They employed a temporal pattern attention (TPA)-based deep learning model to analyze data from both the chaotic external-cavity semiconductor laser (ECL) stage and the final output of the NRNG. The model effectively detected correlations in the ECL stage, but not in the post-processed output, suggesting the NRNG's resistance to predictive modeling. Prior to this, the model's predictive power was validated on a linear congruential algorithm-based RNG. The study also compared the model's prediction accuracy with the baseline probability, aligning with Truong et al.'s approach of using the guessing probability as a comparative metric for min-entropy estimation.

Finally, Haohao Li et al. \cite{li2023improvement} proposed a method for min-entropy evaluation using a pruned and quantized deep neural network. They developed a temporal pattern attention-based long short-term memory (TPA-LSTM) model, which they then optimized through pruning and quantization. This optimized model was retrained and tested on various simulated datasets with known min-entropy values. Their results demonstrated greater accuracy in min-entropy estimation compared to NIST SP 800-90B predictors. This study also investigated why NIST predictors often underestimate min-entropy, attributing it to the sensitivity of local predictability probability to parameter variations. This work parallels Yang et al. and Lv et al.'s in comparing neural network-based min-entropy estimations with NIST SP 800-90B's predictors.
\\

\subsection{Autoregressive Inference and Multi-Token Prediction Strategies}\label{multitoken_theory}

In autoregressive inference, various sampling strategies can be employed to generate sequences, such as greedy decoding \cite{radford2019language}, beam search \cite{vijayakumar2016diverse,shao2017generating}, top-k sampling \cite{fan2018hierarchical,radford2019language} or top-p/nucleus sampling \cite{holtzman2019curious}, with or without temperature-based sampling techniques \cite{ackley1985learning}. However, these techniques may not always yield the globally optimal sequence, as they rely on local decisions at each step.

Our goal is to approximate the global maximum probability for the complete sequence of $n$ bits, given the previous $p$ bits:

\begin{equation}
\max_{x_t, \ldots, x_{t+n}} P(x_t, \ldots, x_{t+n} \mid x_{t-1}, \ldots, x_{t-p}) \ .
\end{equation}

To illustrate the potential limitations of autoregressive inference strategies, let us consider greedy decoding as an example. Greedy decoding selects the most probable bit at each step, conditioned on the previously generated bits. This can be expressed as:

\begin{equation*}
 \resizebox{\hsize}{!}{
      $
\begin{aligned}
    \prod_{k=t}^{t+n} \max_{x_k} P(x_k \mid x_{k-1}, \ldots, x_{k-p}) \bigg\rvert_{
\begin{array}{c}
x_{k-i} = \arg\max_{x_{k-i}} P(x_{k-i} \mid x_{k-i-1}, \ldots, x_{k-i-p}), \\
\forall i \in [1, k \leq p]
\end{array} \ .
}
\end{aligned}
$}
\end{equation*}

However, the product of the maximum conditional probabilities at each step does not necessarily equal the global maximum probability over the entire sequence. In other words, the greedy decoding approach may lead to suboptimal sequences, as it does not consider the joint probability of the complete sequence.

While other search methods, such as beam search, top-k sampling, or top-p/nucleus sampling, can perform better than greedy decoding, they still face the same fundamental challenge. Ultimately, the effectiveness of these methods in approximating the global maximum depends on the data and the search space. As the sequence length $n$ increases, the search space grows exponentially, making it increasingly difficult to find the globally optimal sequence efficiently.

Recently, incorporating future information into language generation tasks has gained attention. Li et al. (2017) \cite{li2017learning} proposed an actor-critic model that integrates a value function to estimate future success, combining MLE-based learning with an RL-based value function during decoding. Oord et al. (2018) \cite{oord2018representation} aimed to preserve mutual information between context and future tokens by modeling a density ratio, rather than directly predicting future tokens. Serdyuk et al. (2018) \cite{serdyuk2017twin} addressed the challenge of long-term dependency learning in RNNs by running forward and backward RNNs in parallel to better capture future information. Lawrence et al. (2019) \cite{lawrence2019attending} trained an encoder by concatenating source and target sequences and using placeholder tokens in the target sequence, which are replaced during inference to generate the final output. These advancements illustrate the growing interest in and potential for optimizing future token predictions in natural language processing tasks.

Qi et al. (2020) \cite{qi2020prophetnet} introduce ProphetNet, a sequence-to-sequence pre-training model that employs a novel self-supervised objective called future n-gram prediction and an n-stream self-attention mechanism. Unlike traditional sequence-to-sequence models that optimize one-step-ahead prediction, ProphetNet predicts the next $n$ tokens simultaneously based on previous context tokens at each time step. This approach explicitly encourages the model to plan for future tokens and prevents overfitting on strong local correlations. The authors pre-train ProphetNet using base and large-scale datasets and demonstrate state-of-the-art results on abstractive summarization and question generation tasks.

Recent works have explored the use of multi-token prediction to improve the efficiency and performance of large language models. Gloeckle et al. propose a memory-efficient implementation and demonstrate the effectiveness of this approach on various tasks, showcasing strong performance on summarization, speeding up inference by a factor of 3×, and promoting the learning of longer-term patterns \cite{gloeckle2024better}. This method has been shown to improve sample efficiency, downstream capabilities, and inference speed, especially for larger model sizes and generative benchmarks like coding. Similarly, Stern et al. \cite{stern2018blockwise} and Cai et al. \cite{cai2024medusa} introduce methods that augment LLM inference by adding extra decoding heads to predict multiple subsequent tokens in parallel. Cai et al. refine the concept introduced by Stern et al. and propose MEDUSA, which uses a tree-based attention mechanism to construct and verify multiple candidate continuations simultaneously. While all three approaches leverage multi-token prediction, Gloeckle et al. focus on the effects of such a loss during pretraining, whereas Stern et al. and Cai et al. propose model finetunings for faster inference without studying the pretraining effects \cite{gloeckle2024better}.

\section{Theoretical Framework}
\label{theoretical_framework_section}

In this section, we establish the theoretical framework for our study. We focus on investigating the concept of average min-entropy and its relationship with min-entropy, particularly within the context of gbAR(p) models. The proofs and auxiliary results supporting our findings can be found in the Appendix.

\subsection{Entropies}
\label{entropies_subsection}

The different entropies that will be considered throughout this paper are gathered in the following:
\begin{definition}\label{entropies_definition}
    Let $\{X_t\}_{t\in\mathbb{Z}}$ be a stochastic process with discrete state-space and let $(X_{t-p},\ldots,X_{t-1},X_t,\ldots, X_{t+n})$ be a subset of $\{X_t\}_{t\in\mathbb{Z}}$ where $t,p,n\in\mathbb{Z}$ and $p,n$ are non-negative. Then:
    \begin{enumerate}
        \item[a)] The min-entropy of $(X_t,\ldots, X_{t+n})$ is:
        \begin{equation*}
           \resizebox{0.9\hsize}{!}{$
    H_{\infty}(X_t,\ldots, X_{t+n})=-\log \left[\max_{x_t,\ldots,x_{t+n}}P(x_t,\ldots,x_{t+n})\right].
        $}
\end{equation*}
         \item[b)] The min-entropy per bit of $(X_t,\ldots, X_{t+n})$ is:
                 \begin{equation*}
    h_{\infty}(X_t,\ldots, X_{t+n})= \frac{1}{n+1}H_{\infty}(X_t,\ldots, X_{t+n}).
    \end{equation*}
          \item[c)] The min-entropy per bit of $\{X_t\}_{t\in\mathbb{Z}}$ is:
          \begin{equation*}
                  \resizebox{0.9\hsize}{!}{
      $
           \begin{aligned}
                   h_{\infty}(\{X_t\}_{t \in \mathbb{Z}}) = - \lim_{k\to \infty}\frac{1}{k+1} \log \left[\max_{x_t,\mid t \mid  \leq \mid k/2  \mid} P(\{x_t\}_{\mid t \mid  \leq \mid k/2  \mid})\right] \ ,\\
                   k \in \mathbb{Z}.
           \end{aligned}   
               $}
\end{equation*}
\item[d)] The worst-case min-entropy of $(X_t,\ldots, X_{t+n})$ is:
\begin{equation*}
\begin{aligned}
        H_{\infty}(X_t,\ldots, X_{t+n}\mid X_{t-1},\ldots, X_{t-p})=
   \\ -\log \left[\max_{x_{t-p},\ldots,x_{t+n}}P(x_t,\ldots,x_{t+n}\mid x_{t-1},\ldots, x_{t-p})\right] .
\end{aligned}
\end{equation*}
\item[e)] The worst-case min-entropy per bit of $(X_t,\ldots, X_{t+n})$ is:
\begin{equation*}
\begin{aligned}
        h_{\infty}(X_t,\ldots, X_{t+n}\mid X_{t-1},\ldots, X_{t-p})
   =\\ \frac{1}{n+1} H_{\infty}(X_t,\ldots, X_{t+n}\mid X_{t-1},\ldots, X_{t-p}).
\end{aligned}
\end{equation*}
           \item[f)] The average min-entropy  of 
            $(X_t,\ldots, X_{t+n})$  is:
           \begin{equation*}
            \resizebox{0.9\hsize}{!}{
      $
        \begin{aligned}
        \tilde{H}_{\infty}(X_t,\ldots, X_{t+n}\mid X_{t-1},\ldots, X_{t-p}) =\\ - \log \left( \left\langle \max_{x_{t},\ldots,x_{t+n}} P( x_{t},\ldots,x_{t+n} \mid x_{t-1},\ldots,x_{t-p}) \right\rangle_{x_{t-1},\ldots,x_{t-p}} \right)     = 
        \\-\log \left[\sum_{x_{t-1},\ldots, x_{t-p}}P(x_{t-1},\ldots, x_{t-p})\max_{x_t,\ldots,x_{t+n}}P(x_{t},\ldots, x_{t+n}\mid x_{t-1},\ldots, x_{t-p})\right] .
        \end{aligned}
        $}
    \end{equation*}

            \item[g)] The average min-entropy per bit of $(X_t,\ldots, X_{t+n})$ is:
\begin{equation*}
\begin{aligned}
        \tilde{h}_{\infty}(X_t,\ldots, X_{t+n}\mid X_{t-1},\ldots, X_{t-p})=\\ \frac{1}{n+1}\tilde{H}_{\infty}(X_t,\ldots, X_{t+n}\mid X_{t-1},\ldots, X_{t-p}).
\end{aligned}
\end{equation*}    
    \end{enumerate}
\end{definition}
\begin{remark}
    When determining the min-entropy of an entire binary stochastic process $\{X_t\}_{t \in \mathbb{Z}}$, the direct evaluation
\begin{equation*}
    H_{\infty}(\{X_t\}_{t \in \mathbb{Z}})=-\log \left[\max_{x_t,t \in \mathbb{Z}} P(\{x_t\}_{t \in \mathbb{Z}})\right] 
    \end{equation*}
can lead to undefined behaviour. Indeed, if we write this as the limit
\begin{align*}
    H_{\infty}(\{X_t\}_{t \in \mathbb{Z}}) = \lim_{k \to \infty} H_{\infty}(\{X_t\}_{\mid t \mid  \leq \mid k/2 \mid}) \ , k \ \text{even}
\end{align*}
the maximum probability decay with k is bounded below by $\frac{1}{2^{k+1}}$ corresponding to the uniform noise, so in that case the limit
\begin{align*}
    H_{\infty}(\{X_t\}_{t \in \mathbb{Z}}) = - \lim_{k \to \infty} \log \left[\frac{1}{2^{k+1}}\right] = 1 + \lim_{k \to \infty} k
\end{align*}
diverges, growing as $\sim k$ with the number of elements $k$. Then the limit
\begin{equation*}
                \resizebox{0.9\hsize}{!}{
      $
\begin{aligned}
        h_{\infty}(\{X_t\}_{t \in \mathbb{Z}}) = - \lim_{k\to \infty}\frac{1}{k+1} \log \left[\max_{x_t,\mid t \mid  \leq \mid k/2  \mid} P(\{x_t\}_{\mid t \mid  \leq \mid k/2  \mid})\right] \ ,\\  k \in \mathbb{Z}
\end{aligned}
$}
\end{equation*}
is bounded by 1 for all $\{X_t\}_{t \in \mathbb{Z}}$. For this reason we are going to refer to this as the min-entropy per bit of the stochastic process $\{X_t\}_{t \in \mathbb{Z}}$.
\end{remark}

\subsection{Order-p Markov Chains}\label{Order-p Markov chains}

Let us begin by defining order-$p$ Markov Chains:

\begin{definition}[cf.  \cite{ching2006markov}, \cite{raftery1985model}]
\label{definition_markov}
    An order-$p$ Markov Chain is a  stochastic process $\{X_t\}_{t\in\mathbb{Z}}$ with discrete state-space $S$ such that:
    \begin{equation*}
        P(x_{t_0}\mid \{x_t\}_{t< t_0})=P(x_{t_0}\mid x_{t_0-1},\ldots, x_{t_0-p})
    \end{equation*}
    for every $t_0\in\mathbb{Z}$ and every $x_t\in S$ such that $t\leq t_0$.
\end{definition}

The aim of this subsection is to define special types of Markov chains and to prove some formulas that are applicable to them regarding the  entropies of Definition \ref{entropies_definition}. Although the experiments performed in this paper mostly involve Generalized Binary Autoregressive Models (see Definition \ref{definition_gbarp} below), other types of Markov chains that have connections with Generalized Binary Autoregressive Models are also explored (see Figure \ref{ModelHierarchy}) because they share certain properties with them and entropy formulas that are interesting on their own can be derived with relatively little additional effort for these processes. 

\begin{definition}
    Let $\{X_t\}_{t\in\mathbb{Z}}$ be an order-$p$ Markov chain with state-space $S$. Then:
    \begin{enumerate}
        \item[i)] $\{X_t\}_{t\in\mathbb{Z}}$ is said to be binary if $S=\{0,1\}$.
        \item[ii)] $\{X_t\}_{t\in\mathbb{Z}}$ is said to be stationary if      
        \begin{equation*}
                        \resizebox{0.9\hsize}{!}{
      $
             P(X_{t_1}=x_{1},\ldots, X_{t_n}= x_{n})=
             P(X_{t_1+\tau}=x_{1},\ldots, X_{t_n+\tau}= x_{n})
    $}
    \end{equation*}
    for every $\tau, t_1,\ldots, t_n\in\mathbb{Z}$, every $x_{1},\ldots,x_{n}\in S$ and every positive integer $n$.
    \item[iii)] $\{X_t\}_{t\in\mathbb{Z}}$ is said to have lag-p point-to-point correlations if $$P(x_t\mid x_{t-1},\ldots,x_{t-p})=P(x_t\mid x_{t-p}) \text{ for every $t\in\mathbb{Z}$.}$$
     \item[iv)] $\{X_t\}_{t\in\mathbb{Z}}$ is said to be irreducible if it is stationary, $S$ is finite and  for every $x, x_1,\ldots,x_{p}\in S$ there exists a non-negative integer $k$ such that $$P(X_{t+k}=x\mid X_{t-1}=x_1,\ldots,X_{t-p}=x_p)>0.$$
        \item[v)] $\{X_t\}_{t\in\mathbb{Z}}$ is said to be aperiodic if it is stationary, $S$ is finite and for every $x\in S$, $$\gcd\{n\geq 1:P(X_{t+n}=x\mid X_t=x)>0\}=1.$$
    \end{enumerate}
\end{definition}

\begin{remark}\label{stationaryminentropy}
    Note that if $\{X_t\}_{t\in\mathbb{Z}}$ is a stationary order-$p$ Markov chain then:
    \begin{equation*}
       h_{\infty}(\{X_t\}_{t \in \mathbb{Z}}) = \lim_{n\to \infty} \frac{1}{n+1} H(X_t,\ldots, X_{t+n}).
    \end{equation*}
\end{remark}

\subsubsection{Some Min-Entropy Inequalities for Order-p Markov Chains}\label{some_inequalities}

Here we establish several inequalities involving min-entropy, average min-entropy, and worst-case min-entropy.

We start by noting that for a fixed $n$, the following inequality between min-entropy, average min-entropy, and worst-case min-entropy holds.

\begin{lemma}
    \label{avg_conditional_min_entropy_bound_0}
 Let $\{X_t\}_{t\in\mathbb{Z}}$ be an order-$p$ Markov chain. Then:
 \begin{equation*}
                 \resizebox{0.9\hsize}{!}{
      $
        \begin{aligned}
                h_{\infty}(X_t,\ldots, X_{t+n}) \geq  \tilde{h}_\infty(X_t, \ldots, X_{t+n} \mid X_{t-1}, \ldots, X_{t-p}) \geq \\  h_\infty(X_t, \ldots, X_{t+n} \mid X_{t-1}, \ldots, X_{t-p})\ .
                   \end{aligned}   
        $}
    \end{equation*} 
\end{lemma}

We conclude this part of Section \ref{some_inequalities} with a result regarding order-$p$ Markov chains that establishes a form of monotonicity for their average min-entropy.

\begin{lemma}
\label{increasing_avg_cond_min_entropy}
    Let $\{X_t\}_{t\in\mathbb{Z}}$ be an order-p Markov chain. Then 
    \begin{equation*}
        \begin{aligned}
             \tilde{H}_{\infty}(X_t,\ldots, X_{t+n}\mid X_{t-1},\ldots, X_{t-p})\leq\\  \tilde{H}_{\infty}(X_t,\ldots, X_{t+n+m}\mid X_{t-1},\ldots, X_{t-p}).
        \end{aligned}
    \end{equation*}
\end{lemma}

This lemma establishes that the average min-entropy of an order-$p$ Markov chain is non-decreasing as the length of the future sequence, $n$, increases. This property reflects the intuitive notion that the uncertainty about future states cannot decrease when considering longer future sequences. This result will be particularly useful later when we discuss an interesting property of Generalized Binary Autoregressive Models (see Remark \ref{avg_less_or_more_min}).

\subsubsection{Convergence Theorem for the Min-Entropy and Average Min-Entropy of \texorpdfstring{order-$p$}{op} Markov chains}

The purpose of the following results, which is materialized in Theorem \ref{limitavgminentropy} below, is to establish conditions under which an asymptotical equivalence between the average min-entropy and the min-entropy of an order-$p$ Markov Chain can be guaranteed. 

\begin{theorem}[Convergence Theorem {\cite{bozorgmanesh2016convergence}}]\label{ConvergenceTheorem}
     Let $\{X_t\}_{t\in\mathbb{Z}}$ be an irreducible and aperiodic stationary order-$p$ Markov chain with finite state-space $S$. Then for every $x,x_{t-1},\ldots,x_{t-p}\in S$:
     \begin{equation*}
                         \resizebox{\hsize}{!}{
      $
\lim_{n\to \infty}P(X_{t+n}=x\mid X_{t-1}=x_{t-1},\ldots, X_{t-p}=x_{t-p})=P(X_t=x).
               $}
     \end{equation*}
    
\end{theorem}

Building upon this theorem, we can now establish the asymptotic equivalence between the min-entropy and the average min-entropy for order-p Markov chains satisfying certain conditions.

\begin{theorem}\label{limitavgminentropy}
      If $\{X_t\}_{t\in\mathbb{Z}}$ satisfies the hypothesis of the Convergence Theorem, i.e. $\{X_t\}_{t\in\mathbb{Z}}$ is an irreducible and aperiodic stationary order-$p$ Markov
chain with finite state-space, then 
       \begin{equation*}
        \resizebox{\hsize}{!}{
      $
        \begin{aligned}
    h_{\infty}(\{X_t\}_{t\in\mathbb{Z}})=\lim_{n\to \infty} \tilde{h}_\infty(X_t, \ldots, X_{t+n} \mid X_{t-1}, \ldots, X_{t-p})=\\
    \lim_{n\to \infty} h_\infty(X_t, \ldots, X_{t+n} \mid X_{t-1}, \ldots, X_{t-p}).
        \end{aligned}      
        $}
    \end{equation*}   
\end{theorem}

This theorem shows that having conditional information about the process provides no advantage asymptotically under the stated conditions, as the min-entropy and the average min-entropy converge to the same value.

\subsubsection{State-Independent Maximum Transition Probability and Bitflip Symmetric \texorpdfstring{Order-$p$}{Op} Markov Chains}

This section introduces two related classes of Markov chains: State-Independent Maximum Transition Probability (SIMTP) and Bitflip Symmetric Order-$p$ Markov Chains. We investigate the properties of these chains, with a particular focus on their average min-entropy behavior. Bitflip symmetric chains are of interest as they could represent a physical symmetry of the random number generator, such as the symmetry between the two polarization states of a quantum random number generator (QRNG). Additionally, the SIMTP property enables us to perform exact min-entropy calculations for the process.

\begin{definition}[State-Independent Maximum Transition Probability Order-$p$  Markov Chain]
  A stationary order-$p$ Markov chain with state-space $S$  is said to be a State-Independent Maximum Transition Probability (SIMTP) Markov Chain if it satisfies the following property:
  \begin{equation*}
  \begin{aligned}
        \max_{x_{t} \in S} P(x_{t} \mid y_{t-1},  \ldots, y_{t-p}) = \max_{x_{t} \in S} P(x_{t} \mid z_{t-1}, \ldots, z_{t-p}) \\
        \text{for every } y_{t-1}, \ldots, y_{t-p}, z_{t-1}, \ldots, z_{t-p} \in S.
  \end{aligned}          
  \end{equation*}
\end{definition}

SIMTP models are those stationary Markov chains for which the maximum transition probability is independent of the initial state sequence of length $p$ in the chain.\\

\begin{proposition}
\label{entropies_SIMTP}
Let $\{X_t\}_{t\in\mathbb{Z}}$ be an order-$p$ SIMTP model with state-space $S$. Then, for every non-negative integer $n$ and every $x_{t-1},\ldots, x_{t-p}\in S$:
\begin{equation*}
\begin{aligned}
      h_{\infty}(\{X_t\}_{t \in \mathbb{Z}})=\tilde{h}_\infty(X_t, \ldots, X_{t+n} \mid X_{t-1}, \ldots, X_{t-p}) = \\ - \log \left[ \max_{x_t} P(x_t\mid x_{t-1}, \ldots , x_{t-p})\right].  
\end{aligned}
\end{equation*}
\end{proposition}

Hence, the min-entropy of the SIMTP process can be computed straightforwardly from its transition probability.

\begin{definition}[Bitflip Symmetry in Binary Order-$p$  Markov Chains]
  A binary order-$p$  Markov chain  exhibits \textit{Bitflip Symmetry} if for all states $x_{t-p}, \ldots, x_{t-1}, x_{t}, \ldots, x_{t+n} \in \{0,1\}$ and for all non-negative integer \( n \) the following property holds:
\begin{equation*}
\begin{aligned}
          P(x_{t}, \ldots, x_{t+n} \mid x_{t-1}, \ldots, x_{t-p}) = \\ P(1 \oplus x_{t}, \ldots, 1 \oplus x_{t+n} \mid 1 \oplus x_{t-1}, \ldots, 1 \oplus x_{t-p})
\end{aligned}
\end{equation*}
  where $ \oplus$ represents the XOR operation. 
\end{definition}

Bitflip symmetric order-$p$ Markov chains are those binary order-$p$ Markov chains for which flipping the bits of all the variables in the conditional probability statement does not change the transition probability. These chains do not distinguish between 0 and 1 but still exhibit some correlation. Our interest in Bitflip symmetric order-p Markov chains is due to the following:

\begin{lemma}
\label{bitflip_symmetric_simtp}
Let $\{X_t\}_{t\in\mathbb{Z}}$ be an order-$p$ Bitflip-Symmetric Markov Chain with lag-p point-to-point correlations. Then   $\{X_t\}_{t\in\mathbb{Z}}$ is a SIMTP order-$p$ Markov chain. 
\end{lemma}

\subsubsection{Generalized Binary Autoregressive Models}
\label{Generalized Binary Autoregressive Models}

The gbAR(p) model \cite{jentsch2019generalized} is an autoregressive (AR) model for binary time series data. It allows the autoregressive parameters to take values in the range (-1, 1), enabling the model to capture negative autocorrelations and alternating patterns. Despite this flexibility, the gbAR(p) model maintains a parsimonious parameterization, making it a compact yet powerful model for binary data.
The gbAR(p) model is a parsimonious subclass of p-th order Markov chains for binary data. While sacrificing some flexibility compared to a full p-th order Markov chain, the gbAR(p) model offers a much more compact representation.

\begin{definition}[Generalized Binary Autoregressive Models \cite{jentsch2019generalized}]
\label{definition_gbarp}
Given $t\in\mathbb{Z}$ let $\left(a_t^{(1)},\ldots,a_t^{(p)},b_t\right)\sim M(1;|\alpha_1|,\ldots,|\alpha_p|,\beta)$ for some $\alpha_1,\ldots,\alpha_p\in (-1,1),\beta\in(0,1]$ such that: $$\sum_{i=1}^p|\alpha_i| + \beta=1 $$ and let $e_t\sim B(1,\epsilon_t)$ for some $\epsilon_t\in(0,1)$.
A stationary binary order-p Markov chain $\{X_t\}_{t\in\mathbb{Z}}$ that can be written in operator form as 
\begin{equation*}
 \resizebox{\hsize}{!}{
      $
    X_t = \sum_{i=1}^p \left[ a_t^{(i)} \mathds{1}_{\lbrace \alpha_i \geq 0 \rbrace} (0 \oplus \cdot ) +  a_t^{(i)} \mathds{1}_{\lbrace \alpha_i < 0 \rbrace} (1 \oplus \cdot )  \right] X_{t-i} + b_t e_t
    $}
\end{equation*}

where $\mathds{1}$ is the indicator function and $\oplus$ is the XOR gate is said to be a  Generalized Binary Autoregressive or gbAR(p) model.
\end{definition}

We will denote the array of coefficients $\alpha_1,\ldots,\alpha_p$ as $\boldsymbol{\alpha}$, and its L1 norm (i.e. the sum of the absolute values of its components) as 
$|\boldsymbol{\alpha}|$.

Our experiments are performed on data generated from gbAR(p) models. The rest of this section is devoted to define the type of gbAR(p) models we will be most interested in, to prove they satisfy the hypothesis of the Convergence Theorem \ref{ConvergenceTheorem}  and to obtain a formula for their min-entropy that will allow to evaluate  machine learning predictors entropy estimations (see Proposition \ref{gbarpu+convergence} and Proposition \ref{min_entropy_gbarp} below).

\begin{definition}
\label{positive_uniform_noise_gbarp}
    Let $\{X_t\}_{t\in \mathbb{Z}}$ be a gbAR(p) model. Then:
    \begin{enumerate}
        \item[i)]  $\{X_t\}_{t\in \mathbb{Z}}$ is said to be positive if $\alpha_i\geq 0$ for every $i\in\{1,\ldots,p\}$.
        \item[ii)] $\{X_t\}_{t\in \mathbb{Z}}$ is said to be a uniform noise gbAR(p) model if $e_t\sim B\left(1,\frac{1}{2}\right)$ for every $t\in\mathbb{Z}$.
    \end{enumerate}
\end{definition}
 Special attention will be paid to uniform noise and positive gbAR(p) models.

\begin{proposition}\label{min_entropy_gbarp}
     Let $\{X_t\}_{t\in\mathbb{Z}}$ be a uniform noise and positive gbAR(p) model. Then
       \begin{equation*}
        \resizebox{\hsize}{!}{
      $
        \begin{aligned}
    h_{\infty}(\{X_t\}_{t\in\mathbb{Z}})=\lim_{n\to \infty} \tilde{h}_\infty(X_t, \ldots, X_{t+n} \mid X_{t-1}, \ldots, X_{t-p})=\\
    \lim_{n\to \infty} h_\infty(X_t, \ldots, X_{t+n} \mid X_{t-1}, \ldots, X_{t-p})=-\log\left(1-\frac{\beta}{2}\right).
        \end{aligned}   
        $}
    \end{equation*}   
\end{proposition}

\begin{remark}
\label{avg_min_entropy_limit_match}
    Let $\{X_t\}_{t\in\mathbb{Z}}$ be a uniform noise gbAR(p) model with point-to-point lag-$p$ correlations. Apart from having point-to-point lag-$p$ correlations, $\{X_t\}_{t\in\mathbb{Z}}$
    is bitflip-symmetric by Lemma \ref{gbarp_lag_bitflip}. Hence, $\{X_t\}_{t\in\mathbb{Z}}$ is SIMTP by Lemma \ref{bitflip_symmetric_simtp} and therefore for every non-negative integer $n$ Proposition \ref{entropies_SIMTP} yields
    $$ h_{\infty}(\{X_t\}_{t \in \mathbb{Z}})=\tilde{h}_\infty(X_t, \ldots, X_{t+n} \mid X_{t-1}, \ldots, X_{t-p}).$$
      The argumentation above is illustrated in the first two plots of Figure \ref{avg_vs_mean_alpha_p_n_dependance}, where the equivalence of the average min-entropy per bit and the min-entropy of uniform noise gbAR(p) models with point-to-point lag-$p$ correlations is observed regardless of the values $n,p,\alpha_p$.
\end{remark}

\begin{figure}
\centering
    \includegraphics[width=\columnwidth]{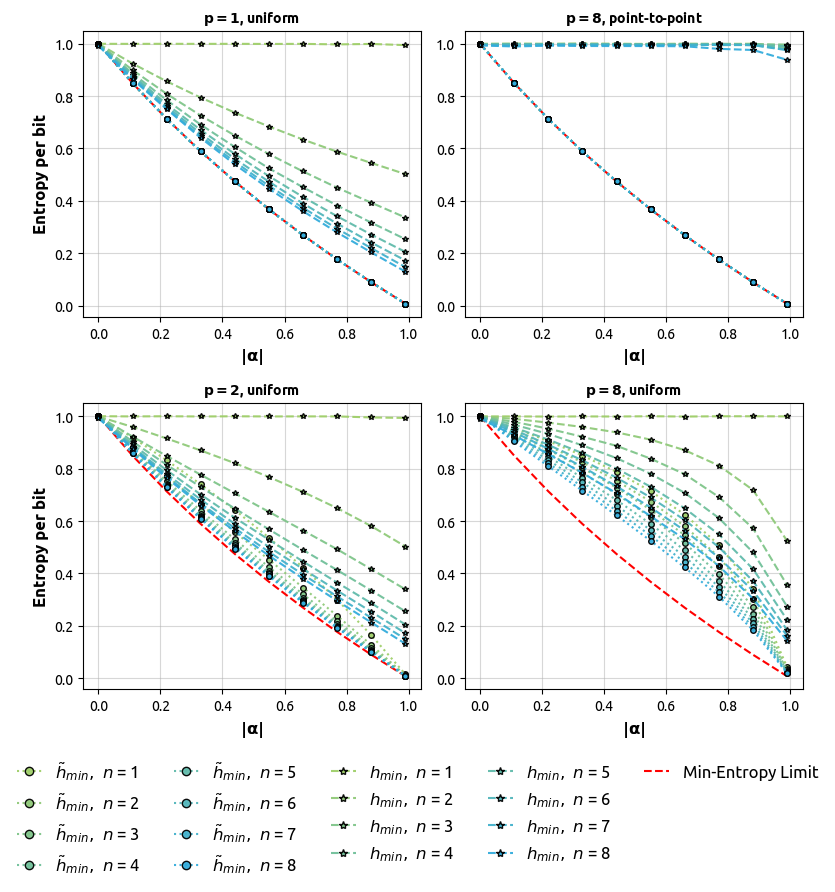}

\caption{$|\boldsymbol{\alpha}|$ dependance of average min-entropy compared with min-entropy and min-entropy limit for several sequence lengths $n$, correlation scales $p$ and autocorrelation functions (\textit{uniform} and \textit{point-to-point}). The data points have been evaluated numerically (see Subsection \ref{Theoretical min-entropies calculation}).} \label{avg_vs_mean_alpha_p_n_dependance}
\end{figure}

\begin{remark}
    Let $\{X_t\}_{t\in \mathbb{Z}}$ be a uniform noise and positive gbAR(p) model. Since $\{X_t\}_{t\in \mathbb{Z}}$ is stationary by Definition \ref{definition_gbarp} and it satisfies the hypothesis of the Convergence Theorem \ref{ConvergenceTheorem} by Proposition \ref{gbarpu+convergence}, it follows that 
            \begin{equation*}
        \lim_{n\to \infty} h(X_t,\ldots, X_{t+n})= h_{\infty}(\{X_t\}_{t \in \mathbb{Z}}) 
    \end{equation*}
    by Remark \ref{stationaryminentropy} and
    \begin{equation*}
              \lim_{n\to \infty}  \tilde{h}_\infty(X_t, \ldots, X_{t+n} \mid X_{t-1}, \ldots, X_{t-p})= 
        h_{\infty}(\{X_t\}_{t\in\mathbb{Z}})
    \end{equation*}
    by Theorem \ref{limitavgminentropy}. Moreover, 
        \begin{equation*}
  h_{\infty}(X_t,\ldots, X_{t+n})\geq \tilde{h}_\infty(X_t, \ldots, X_{t+n} \mid X_{t-1}, \ldots, X_{t-p})
    \end{equation*}
     by Lemma \ref{avg_conditional_min_entropy_bound_0}.
The three (in)equations above are illustrated in Figure \ref{avg_vs_min_n_dependance}, where we can observe that both the  average min-entropy per bit of $(X_t, \ldots, X_{t+n})$ and the  min-entropy per bit of $(X_t, \ldots, X_{t+n})$ tend to the min-entropy of $\{X_t\}_{t\in\mathbb{Z}}$ when $n$ goes to infinity, being the former partial entropy lower than the latter.
\end{remark}

\begin{figure}
\centering
    \includegraphics[width=\columnwidth]{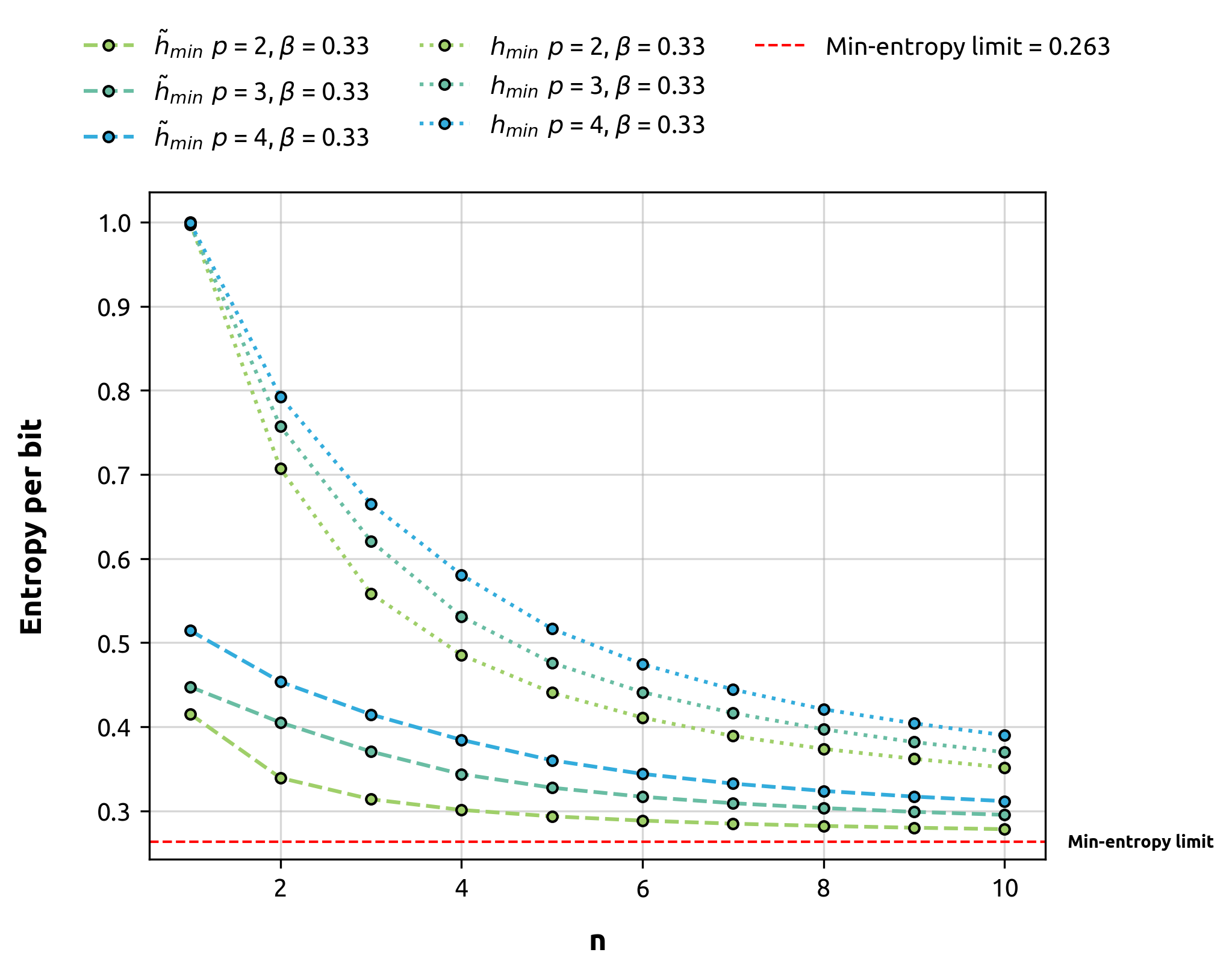}
\caption{Asymptotic behaviour of average min-entropy and min-entropy per bit in terms of the target space size $n$ for gbAR(p) models with several correlation lengths $p$ and fixed $\beta$. The $\boldsymbol{\alpha}$ arrays are uniform (i.e. all their components are equal). The data points have been evaluated numerically (see Subsection \ref{Theoretical min-entropies calculation}).} 
\label{avg_vs_min_n_dependance}
\end{figure}

\begin{remark}
\label{avg_less_or_more_min}
    Let $\{X_t\}_{t\in\mathbb{Z}}$ be an order-$p$ Markov chain. It is worth noting that although the average min-entropy  of $(X_t, \ldots, X_{t+n})$ cannot decrease with $n$ by Lemma \ref{increasing_avg_cond_min_entropy}, the average min-entropy per bit of $(X_t, \ldots, X_{t+n} \mid X_{t-1}, \ldots, X_{t-p})$ can actually do it (see Figure \ref{avg_vs_min_n_dependance}).
\end{remark}

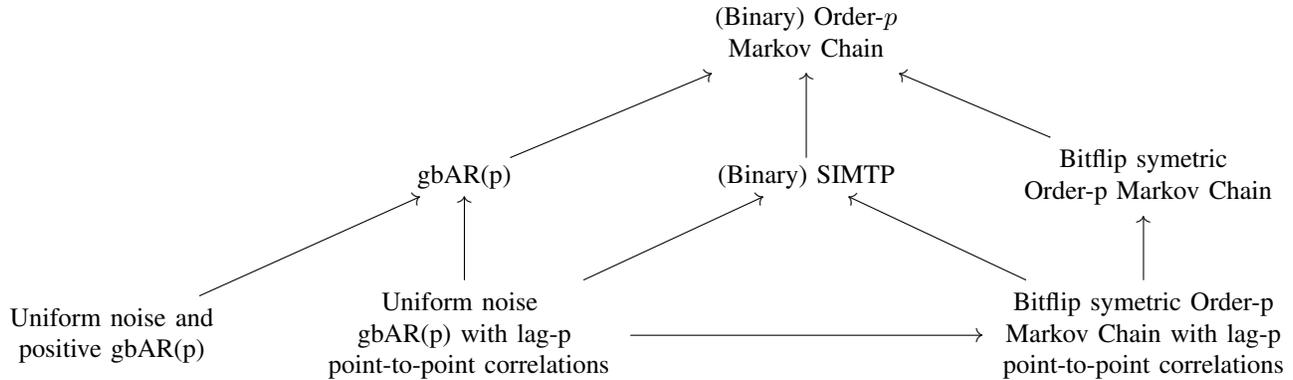
\begin{figure*}[ht]
    \centering
  $$
\xymatrix{ &
 & 
 {\begin{array}{c}
    \text{(Binary) Order-$p$}  \\
    \text{Markov Chain} 
 \end{array}} & \\
 & \text{gbAR(p)} \ar[ur] & \text{(Binary) SIMTP} \ar[u] & {\begin{array}{c}
    \text{Bitflip symetric}\\ \text{ Order-p Markov Chain} 
 \end{array}} \ar[ul]\\
 {\begin{array}{c}
    \text{Uniform noise and}  \\
    \text{positive gbAR(p)} 
 \end{array}} \ar[ur] &
{\begin{array}{c}
    \text{Uniform noise }  \\ \text{gbAR(p) with lag-p} \\
    \text{ point-to-point correlations}
 \end{array}}  \ar[rr] \ar[u] \ar[ur]  & & {\begin{array}{c}
    \text{Bitflip symetric Order-p}\\ \text{Markov Chain with lag-p}  \\
    \text{point-to-point correlations}
 \end{array}} \ar[u] \ar[ul]
}
$$    \caption{Hierarchy of the main models considered in this paper. An arrow from  model $M_1$ to  model $M_2$ means that every model of type $M_1$ is of type $M_2$.}
\label{ModelHierarchy}
\end{figure*}

\section{Experimental Methodology}

In this section, we outline the experimental methodology carried out, which is primarily based on code implementations. Our main goal is to validate our theoretical findings, for which we generate correlated data using gbAR(p) models (see Definition \ref{positive_uniform_noise_gbarp}).

Building upon Kelsey's predictor concept \cite{kelsey2015predictive}, we use machine learning as a tool for min-entropy estimation. Our methodology adopts the traditional machine learning approach, marked by separate training and evaluation phases. This strategy deviates from Kelsey's model of continuous updates, which we consider a non-essential aspect of predictor concepts for min-entropy evaluation. Thus, our methodology, termed as machine learning predictors, streamlines the process by clearly separating these stages, focusing on essential predictive capabilities without the need for constant updates.

Contrasting with the approach in \cite{truong2018machine}, which examines processes failing randomness due to large periods, we focus on processes with shorter-range, bit-level correlations since such correlations could be more similar to the realistic failure modes of physical and hardware-based RNGs, in line with the use of order-$k$  or order-2 Markov chains in \cite{kelsey2015predictive} and \cite{zhu2017analysis} respectively. Given this requirement for modeling realistic RNG failures with shorter-range dependencies, gbAR(p) models provide a parsimonious parameterization that allows us to control correlations and anticorrelations, making them a suitable choice for our analysis.

This data serves as the training set for two distinct types of neural networks, which are tasked with predicting the next \texttt{target\_bits} bits. As highlighted in the theory section (\ref{avg_min_entropy_limit_match} and Figure \ref{avg_vs_mean_alpha_p_n_dependance}), order-1 Markov Chains may present trivial cases where min-entropy and average min-entropy match. Therefore, it is important to analyze the behavior of our predictors in scenarios where this equivalence does not hold, forming the basis of our experimental approach.

All the experiments are conducted on a NVIDIA GeForce RTX 3090 with 24.576 GiB of memory and CUDA Version 12.3.

The experimental framework can be structured around four primary components: the data generation process using the gbAR(p) model, the Monte Carlo simulation for the evaluation of minimum entropies, the implementation of machine learning models training and evaluation, specifically GPT-2 and a variation of RCNN (a model taken from \cite{truong2018machine}), and the integration of all data processing steps. This pipeline encompasses the generation of gbAR(p) data, its evaluation using the NIST SP 800-90B test suite, the execution of machine learning predictions, and the compilation of relevant results.

For detailed documentation on code usage, parameter explanations, and further technical details, please refer to the \texttt{README.md} file in the associated code repository \cite{github_code_repo}.

\subsection{Data Generation}

The data generation is an implementation of the gbAR(p) model (Definition \ref{definition_gbarp}) in the \texttt{gbAR()} function. The call to this function is wrapped in the function \texttt{generate\_gbAR\_random\_bytes()}, which leverages different autocorrelation functions, namely \textit{point-to-point}, \textit{uniform} (all the components being equal), \textit{exponential}, and \textit{Gaussian}, to define the autocorrelation pattern through the $\boldsymbol{\alpha}$ parameter (as defined in Definiton \ref{definition_gbarp}).

It computes binary sequences by considering the autocorrelation defined by $\boldsymbol{\alpha}$ and the (here always uniform) random noise term (weighted by $\beta$), sourced from high-entropy random numbers generated by OpenSSL's \texttt{rand} command as the source of high entropy random numbers. Random numbers from OpenSSL are also employed in the \texttt{ossl\_rand\_mn\_rvs()} function, that generates samples from a multinomial distribution as required by the gbAR(p) model. These samples are then used to construct the final binary sequence according to the autocorrelation characteristics defined by the model parameters.

The \texttt{gbAR()} function includes a mechanism that discards an initial segment (here with size $10^4$ bytes) of the generated binary sequence. The rationale behind this is to allow the sequence to reach a state of statistical stationarity, thereby minimizing initial transient effects introduced in the generation.

\subsection{Min-Entropies Calculation} \label{Theoretical min-entropies calculation}
In our approach to numerically evaluate the average min-entropy and min-entropy of the gbAR(p) processes, we employ a Monte Carlo simulation. This involves creating a program that generates 100 samples, each comprising $10^5$ bytes. These samples are used to empirically estimate the joint frequencies. The computed frequencies form the basis for calculating both the min-entropy and the average min-entropy (using the known transition probabilities specific to the gbAR(p) processes).

\subsection{Machine Learning Predictors}

In this work, we use two distinct machine learning models to tackle the task of predicting binary sequences. The first model is an adaptation of the RCNN, while the second model is based on the GPT-2 architecture.

The selection of the RCNN and GPT-2 architectures is driven by our goal to explore prediction capabilities on binary sequences generated from autoregressive models with short-range correlations. The RCNN, as used by Truong et al., has proven effective in detecting correlations in quantum random number generators under deterministic classical noise influence. Its convolutional and recurrent layers are well suited to capture local patterns and short-term dependencies. In contrast, the transformer-based GPT-2 model, with its self-attention mechanism, offers a different approach. Although originally designed for natural language processing, we adapt it to our binary sequence prediction task to examine how it captures order-$p$ Markov chain characteristics. Using these two models enables us to validate the theoretical finding that machine learning predictors tend to estimate average min-entropy independently of architecture, provided they can learn from the data's correlations.

\subsubsection{Target Space Representation and Inference Strategies}

As discussed in Section \ref{multitoken_theory}, various approaches exist for multi-token prediction in language models. While recent works like Gloeckle et al. \cite{gloeckle2024better}, Stern et al. \cite{stern2018blockwise}, and Cai et al. \cite{cai2024medusa} have shown promising results in natural language processing tasks, our research focuses on a different domain. We aim to explore the relationship between model predictions and the min-entropy of the data, specifically for data with different correlations, rather than natural language.

To address the limitations of autoregressive inference strategies that rely on local decisions at each step, we propose directly predicting the entire sequence of n bits simultaneously. This approach allows us to obtain the global maximum probability for the complete sequence from the model, rather than relying on step-by-step decisions. By doing so, we aim to capture long-range dependencies and avoid the potential pitfalls of greedy, beam search or other methods in finding globally optimal sequences.

Our method involves using different tokenization strategies for input and target spaces:

\begin{itemize}
    \item Input space: We use binary tokenization where each token represents a single bit (0 or 1).
    \item Output space: We employ a tokenization where each token represents $n$ bits, resulting in $2^n$ unique classes.
\end{itemize}

This tokenization approach allows the model to predict groups of bits as single tokens, considering the joint probability of the entire sequence. We believe this method will lead to more accurate and globally optimal predictions, as it forces the model to consider the interdependencies between bits in the sequence.

\subsubsection{Model Training and Evaluation Methodology}

Our training dataset consists of sequences generated from the gbAR(p) model. From these sequences, we extract subsequences of length $\texttt{seqlen} - \texttt{target\_bits}$ bits Both models are adapted to classify over 
$2^n$ classes, corresponding to all possible sequences of  $n$ bits. In this context, each possible combination of $\texttt{target\_bits}$ is treated as a distinct class in the classification task.

We evaluate the model prediction accuracy $P_{ML}$ as 

\begin{equation}
P_{ML} = \frac{n_{correct}}{n_{total}} \ .
\end{equation}

Here, $n_{correct}$ represents the number of correct predictions made by the model, and $n_{total}$ denotes the total number of evaluations conducted. This measure of accuracy serves as a key indicator of the model's performance and its ability to accurately predict future bits based on the training received. The estimated min-entropy will be

\begin{equation}
\label{estimated_min_entropy}
     h_{ML} = - \log(P_{ML}) \ .
\end{equation}

We get a basic approximation of the error using the Wald approximation for the binomial proportion confidence interval, as outlined in \cite{meeker2017statistical}. The propagation of this error yields

\begin{equation}
\label{entropy_error}
    \Delta h_{ML} = \frac{z}{\texttt{target\_bits}} \ln(2)\sqrt{ \frac{\frac{1}{P_{ML}} - 1}{ n_{\text{evals}}  }} \ ,
\end{equation}

where $n_{\text{evals}}$ is the number of evaluation sequences.

In addition to accuracy, to assess the performance of the training procedure during the \textit{development} phase, we have included the evaluation of the binary entropy of the predictions, $P_e$, the proportion of zeros in the prediction, $P_c$, and the loss.

\subsubsection{RCNN Model}

We use model based on the RCNN model from the framework presented in \cite{truong2018machine}. The original implementation combines convolutional and LSTM layers followed by fully-connected layers. Initially, input integers are converted into one-hot vectors. These vectors are then processed through convolutional layers with max-pooling to extract features, which are subsequently handled by the LSTM layer to capture temporal dependencies.

Our adaptation transitions from byte-based input processing to bit sequence handling, accommodating classification over a fixed number of $2^n$ classes, where $n$ is the number of \texttt{target\_bits}. This aligns with the original design intended for classification over fixed $2^n$ classes. The architecture employs an output layer with $2^n$ neurons and softmax activation, allowing for multi-class classification. The categorical cross-entropy loss function is used for training. We use the RMSprop optimizer with a learning rate of 0.005.

Regarding the model architecture, we have slightly modified Truong's model to increase its size, including:

\begin{itemize}
    \item Convolution1D layers with 32, 64, and 128 filters, kernel sizes of 12, 6, and 3 respectively, all using 'relu' activation and 'same' padding.
    \item LSTM layers with 256 and 128 units, featuring return sequences and dropout layers with a rate of 0.2 for regularization.
    \item A final Dense layer with an output size equal to \texttt{target\_bits}, using sigmoid activation.
    
\end{itemize}

This model architecture results in approximately $~ 7.6\cdot10^5$ trainable parameters.

\subsubsection{GPT-2 Model}

The GPT-2 model, referenced in \cite{radford2019language}, is adapted from its typical use in natural language processing as provided by the Hugging Face Transformers library \cite{huggingface_transformers}. This adaptation restructures the model for traditional classification over the possible $2^n$ classes for the next $n$ \texttt{target\_bits}.

For processing the binary sequences, we implement a custom \texttt{BinaryDataset} class. Each data entry in this dataset consists of a binary bit sequence with a length defined by the \texttt{seqlen} parameter. This setup facilitates the mapping of each bit in the input sequence to the next \texttt{target\_bits}, aligning with the classification framework.

In terms of model architecture for the adapted GPT-2, the vocabulary size is set to $2^{target\_bits}$, aligning with the number of classes in our classification framework. The specific configuration of the model includes parameters such as \texttt{n\_positions=512}, \texttt{n\_ctx=512}, \texttt{n\_embd=768}, \texttt{n\_layer=3}, and \texttt{n\_head=3}. This configuration leads to the GPT-2 model having $\sim 21\cdot10^6$  trainable parameters.

For the training phase, we use the RMSprop optimizer with a learning rate of 0.0005. The \texttt{CrossEntropyLoss} loss function is chosen as the loss function.

We incorporate gradient scaling and accumulation in our training approach to enhance memory optimization and computational efficiency, especially important under constrained GPU availability.

\subsection{Pipeline}

We encapsulate in a pipeline the entire data processing for this work, from generating random numbers to saving results.

For each selection of input parameters, the method generates random bytes using the previously described gbAR(p) model. These bytes are saved to a file, which is later used in the data generators within the models to train and evaluate the models. We generate new gbAR(p) sequences for each run to ensure data variability and robustness.

The pipeline runs the NIST SP 800-90B entropy assessment for non i.i.d data in parallel with the model execution over a sample of the generated data (here $10^7$ bytes).

Post-analysis, we meticulously compile various results, including entropy assessments, model parameters, $P_{ML}$ values, execution time, and more, into a CSV file. 

\section{Results}

Our primary objective is to investigate the relationship between the estimated min-entropy and the number of \texttt{target\_bits}. To facilitate this analysis, we focus on low-entropy data for several reasons. Firstly, it ensures that models can effectively learn and capture underlying patterns. Secondly, it enhances the distinction between model predictions and inherent noise, allowing for more robust statistical analysis. In high-entropy scenarios (characterized by small $\alpha$ values), the entropy per bit approaches 1, resembling uniform noise (see Figure \ref{avg_vs_mean_alpha_p_n_dependance}). This proximity to maximum entropy poses challenges in assessing model performance due to overlapping confidence intervals. These intervals often encompass both the maximum entropy value of 1 and the expected theoretical value, which is also close to 1. Consequently, a large number of evaluations would be required to reduce measurement uncertainty and achieve statistical distinguishability. To address these challenges, we employ a generalized binary autoregressive model of order $p=10$ with a uniform $\alpha$ vector and a uniform noise term $\beta = 0.5$. This configuration provides a balance between learnable patterns and stochastic noise, enabling effective extraction of the \texttt{target\_bits} dependence while maintaining statistical significance in our results.

Both the GPT-2 and RCNN models' training data varied based on the number of target bits to be predicted. For tasks involving 1 to 12 target bits, 20 million bytes of raw data were used, equivalent to 125,000 training sequences of 128 bits each. This increased to 30 million bytes (187,499 sequences) for 13 to 15 target bits, and further to 42 million bytes (262,499 sequences) for 16 target bits. In each case, 80\% of the available data was allocated for training, with the remaining 20\% reserved for evaluation.

Our primary objective is to compare the min-entropy estimated by these models against the theoretical calculations and the estimations provided by the NIST SP 800-90B predictors and its overall entropy assessment. The results of these experiments are presented in Figure \ref{gpt2_rcnn_vs_theoretical_min_entropies}.

\begin{figure*}[!t]
\centering
\scalebox{0.70}{
    \includegraphics{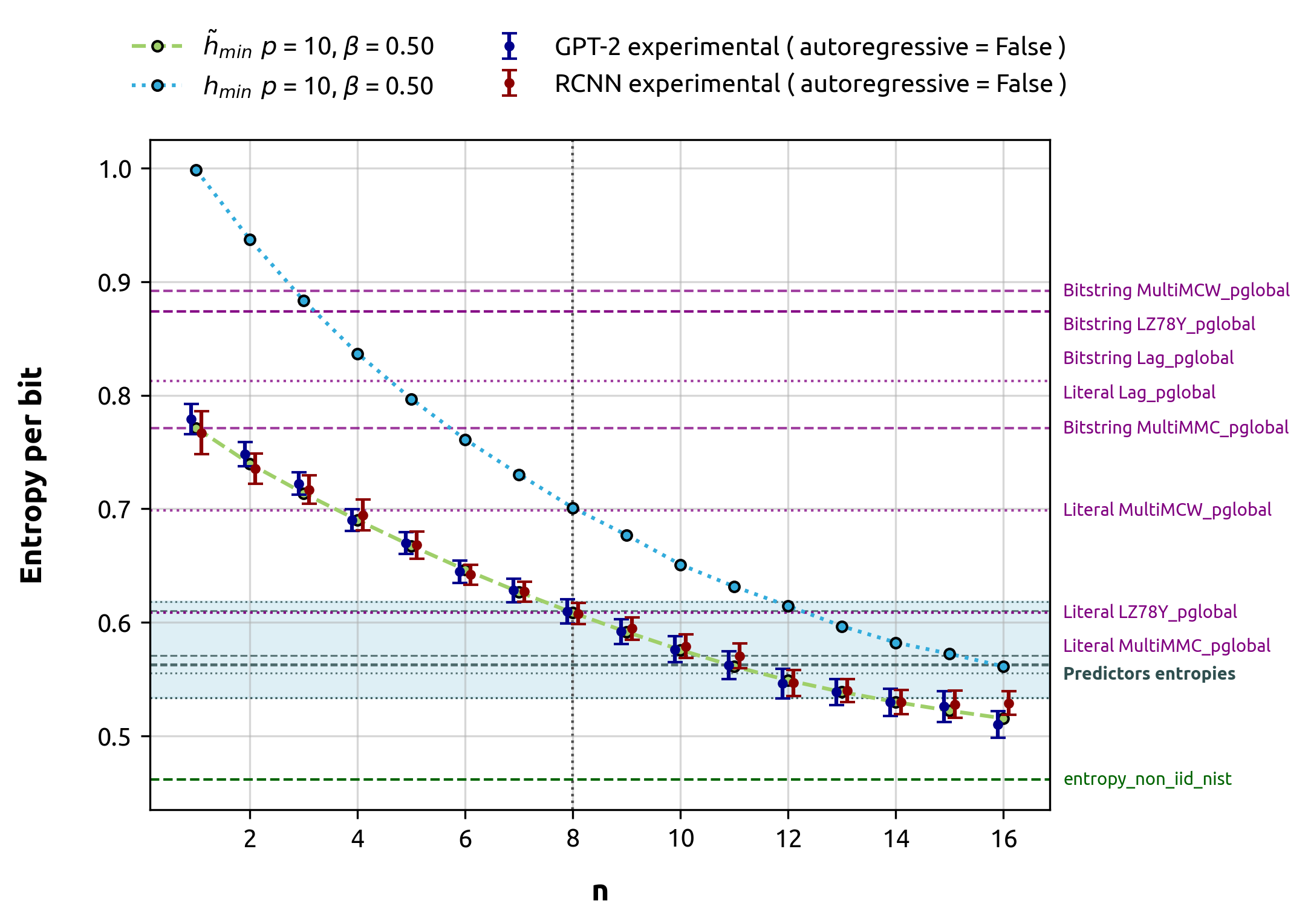}
}
\caption{Theoretical minimum entropies versus machine learning-based estimations of minimum entropy for gbAR(10) with a uniform $\alpha$ vector and a uniform noise term $\beta = 0.5$ (representing a low entropy scenario). For clarity in visualization, the markers representing different models are slightly offset along the x-axis. Error bars are derived using a binomial proportion confidence interval with a 95\% confidence level. The results from the NIST SP 800-90B global predictor tests are emphasized. The highlighted predictor entropies are the minimum of local and global estimates, in this case predominantly influenced by the local estimate. \textit{Bitstring} predictors try to predict the next bit, while \textit{Literal} predictors try to predict the next byte. The \texttt{entropy\_non\_iid\_nist} denotes the final outcome of the NIST SP 800-90B analysis, which is the minimum of all conducted tests in the suite. The \texttt{h\_min\_limit} is the theoretical limit of the min-entropy, interpreted as the min-entropy per bit of the entire process. For this specific gbAR(10) configuration with positive $\alpha$, both min-entropies converge to this limit. }
\label{gpt2_rcnn_vs_theoretical_min_entropies}
\end{figure*}

Finally, for illustration purposes, we demonstrate how greedy decoding fails to accurately estimate the min-entropy in data with certain types of correlations. This experiment utilized 20 million bytes of data, equivalent to 125,000 sequences of 128 bits each. Following our standard protocol, 80\% (100,000 sequences) were used for training, and 20\% (25,000 sequences) for evaluation. In this case, we focused on predicting 1 to 8 target bits. For comparison, we used two gbAR(2) models with alpha vectors $\frac{1}{4}[+1, +1]$ and $\frac{1}{4}[+1, -1]$. In the case with alternating correlation signs the global maximum probability cannot be split as the product of the local maximums, so the greedy decoding leads to suboptimal predictions compared to the inference over $2^n$ classes. In the first case ($|\frac{1}{4}[+1, +1]$), the global maximum can be reached as the product of local maximums at each bit 
(see Remark \ref{transition_probs_gbAR}), so both approaches match. This comparison illustrates how the greedy approach may lead to suboptimal predictions, as it does not consider the joint probability of the complete sequence in all cases. The results of this analysis are illustrated in Figure 5, highlighting how the greedy decoding strategy can fall short in accurately estimating min-entropy under certain correlation conditions, while performing adequately in others.

\begin{figure}
\centering
\includegraphics[width=\columnwidth]{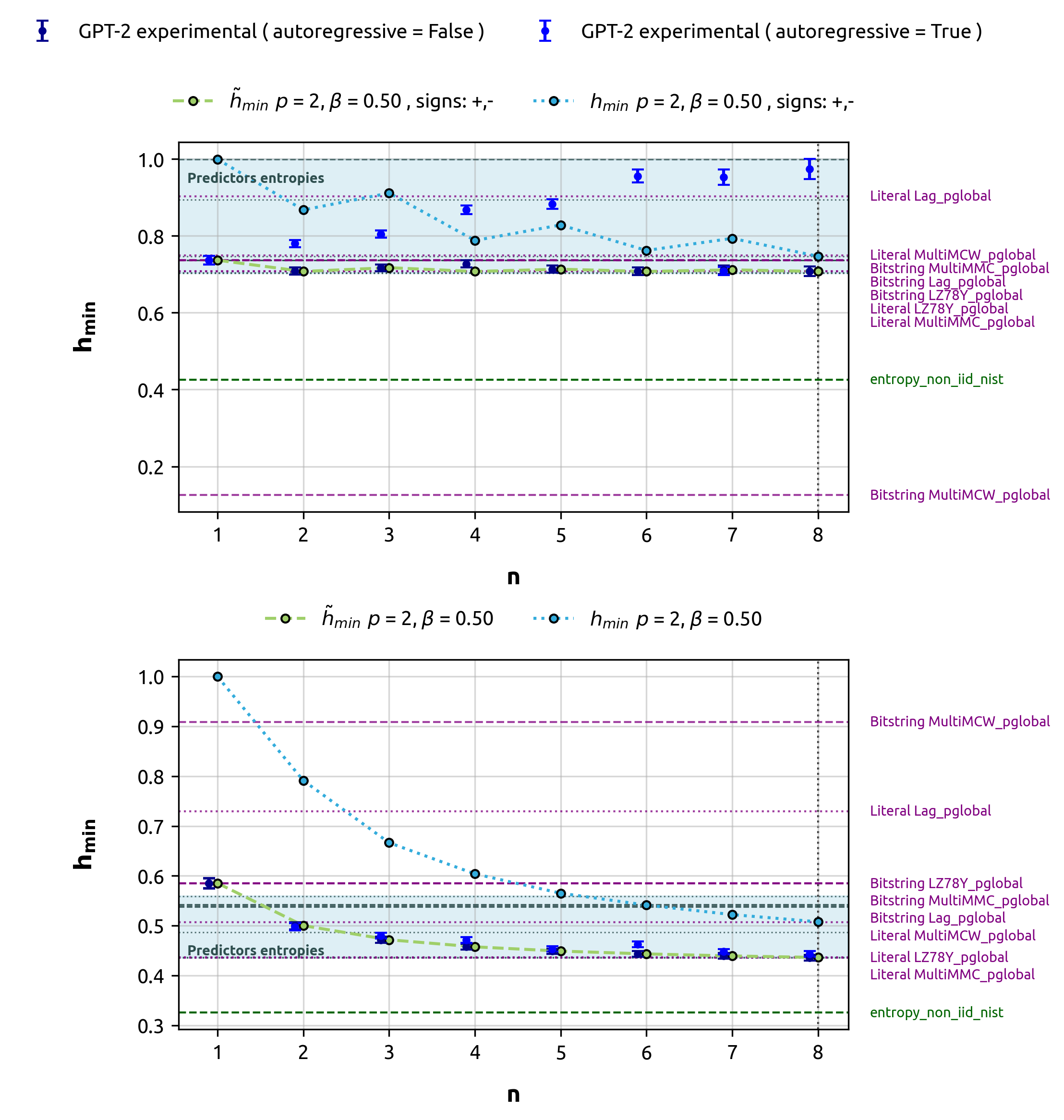}
\caption{Comparison of min-entropy estimates: Greedy decoding vs. direct prediction over $n$ $target\_bits$ for gbAR(2) models. The discrepancy between the experimental estimate of min-entropy and the theoretical is evident for $|\sqrt{\alpha}|[+1, -1]$, while results align for $|\sqrt{\alpha}|[+1, +1]$.} 
\label{autoregressive_inference}
\end{figure}

\section{Discussion}

Our work builds upon Kelsey’s definition of predictors \cite{kelsey2015predictive}, showing that these predictors effectively estimate the average min-entropy as long as they can harness correlations and effectively model conditional, rather than joint, probabilities. This distinction becomes significant when dealing with stochastic processes with complex correlation structures.

We show that while min-entropy varies with the number of bits considered, defining min-entropy per bit for the entire process is still possible. Lemma \ref{avg_conditional_min_entropy_bound_0} establishes that the average min-entropy per bit is always lower than or equal to the min-entropy for order-p Markov chains. Although generally distinct, in specific cases (Theorem \ref{limitavgminentropy}), both joint min-entropy and average min-entropy converge towards a process-wide min-entropy per bit. Interestingly, different states may exhibit varied decay laws despite this common limit (Figure \ref{avg_vs_min_n_dependance}), which is operationally significant when attackers have access to correlated data.

Figure \ref{avg_vs_mean_alpha_p_n_dependance} illustrates the interplay between correlation 'width' and 'length' in how average min-entropy approaches the min-entropy limit. This aligns with Remark \ref{avg_min_entropy_limit_match}, where average min-entropy per bit equals the min-entropy of uniform noise gbAR(p) models with point-to-point lag-$p$ correlations, regardless of \texttt{target\_bits}, $p$, and $|\boldsymbol{\alpha}|$ values.

As we approach high entropy limits, all entropy forms converge to the process limit, consistent across various alpha levels for the considered gbAR(p) models. This reaffirms that min-entropy consistently exceeds average min-entropy, as per Lemma \ref{avg_conditional_min_entropy_bound_0}.

Figure \ref{gpt2_rcnn_vs_theoretical_min_entropies} shows that our min-entropy estimations from both models align with the average min-entropy and stay within the error interval. Interestingly, the NIST Bitstring global predictors, designed to estimate the entropy of 1 target bit, generally overestimate the average min-entropy for 1 target bit, with the notable exception of the MultiMMC predictor. For 8 \texttt{target\_bits}, the NIST global predictors, specifically MultiMCW and Lag, tend to overestimate the average min-entropy. However, LZ78Y and MultiMMC show results that are close to the theoretical calculation. It is important to note that NIST predictors are not designed to estimate entropy in the large $n$ limit. In the particular case studied, where the min-entropy continues to decay beyond $n=8$, it is not surprising that the GPT-2 predictor provides a lower estimate in the $n=16$ run. This estimate is compatible with the theoretically expected value and, moreover, does not overlap with the gray area representing the minimum between the local and global estimates. Consequently, the GPT-2 predictor's estimate is closer to the min-entropy of the stochastic process, making it a better and more conservative estimation.

Local predictions consistently dominate the min-entropy estimate across all cases. As a result, the overall outcome of the predictor is determined by the local estimate, since the final entropy estimate is the lesser of local and global predictions. The local predictions fall within the theoretical min-entropy for the 9-14 \texttt{target\_bits} range and are significantly higher than the min-entropy limit of the process. The overall result of the NIST's entropy non-IID test, which is the minimum of all tests in the suite, including both predictors and non-predictors, is closer to the min-entropy limit. In this specific scenario, predictors do not significantly contribute to the overall test.

Our analysis of greedy decoding versus direct prediction (Figure 5) reveals limitations in autoregressive inference approaches. For certain correlation structures (e.g., gbAR(2) with $|\sqrt{\alpha}|[+1, -1]$), greedy decoding fails to capture global maximum probability. This underscores the importance of multi-token prediction for accurately estimating min-entropy in complex correlation structures. Single-token or greedy approaches may lead to suboptimal predictions and min-entropy overestimation, emphasizing the need for methods capturing joint probabilities over multiple tokens.

In conclusion, our machine learning predictors demonstrate more consistent performance compared to the NIST SP 800-90B in estimating average min-entropy for both 1 and 8 \texttt{target\_bits} in low entropy scenarios, while also providing robust estimates for larger values of $n$. This superior performance can be attributed, in part, to the non-parametric nature of ML min-entropy estimation. Unlike traditional methods that often assume specific underlying distributions, ML approaches allow for flexible modeling, making them particularly powerful in capturing complex, non-linear dependencies in the data. This flexibility is especially valuable when dealing with stochastic processes exhibiting intricate correlation structures, as it enables the model to adapt to the data's inherent patterns without being constrained by predetermined statistical assumptions.

High entropy scenarios present additional challenges, potentially requiring larger training runs and models. Augmenting target bits offers improved capture of long-range correlations but at a significant computational cost. As Equation \ref{entropy_error} indicates, maintaining constant error rates with increasing target bits requires exponential growth in evaluations, as $P_{ML} \sim 2^{-\texttt{target\_bits}}$ in high entropy limits.
This accuracy-computation trade-off necessitates careful balancing of long-range correlation capture and practical computational requirements. Future research should focus on developing efficient algorithms or approximation methods to handle larger target bits without prohibitive computational costs, addressing the challenges of multi-token prediction, autoregressive inference limitations, and target bit scaling in entropy estimation.

\section{Conclusions}

Our research has revealed several key insights into the estimation of min-entropy. We have shown that machine learning predictors are good at estimating average min-entropy, as long as they effectively harness correlations by estimating conditional probabilities. This becomes particularly significant in stochastic processes with complex correlation structures, where the difference between average min-entropy and min-entropy is relevant.

Our results highlight that both these entropies depend on the number of \texttt{target\_bits} considered. Given this important role of \texttt{target\_bits}, especially in scenarios with complex correlation structures, it may be operationally significant to include assessments of both average min-entropy and min-entropy for specific \texttt{target\_bits} values relevant to cryptographic (or other) scenarios. Importantly, in the examples studied, we observed that as average min-entropy decays with increasing \texttt{target\_bits}, targeting only a few bits could lead to an overestimate of the min-entropy. This finding underscores the potential risks of relying on limited-scope entropy estimates in cryptographic applications.
While defining min-entropy per bit for the entire process is feasible, and the entropies studied here converge towards this limit (suggesting a lower bound or worst-case scenario), this bound has not been explicitly demonstrated in this work. Therefore, developing effective methods to estimate this limit remains essential.

We have also found that machine learning predictors can beat NIST SP 800-90B predictors estimates in some cases, making them suitable tools to be included in entropy assessment suites.

Our research leaves several avenues open for exploration that may be of particular interest in further studies:

\begin{itemize}
    \item Development of methods for estimating min-entropy in large \texttt{target\_bits} scenarios.
    \item Exploration of the relationship between training data min-entropy, model size, and necessary training data size for accurate min-entropy estimation. This investigation goes beyond aligning model prediction accuracy with theoretical curves; it aims to provide a deeper understanding of model learning capacity at various entropy levels. Such knowledge could inform the appropriate scaling of computational resources and potentially offer improved estimates by considering theoretical bounds on min-entropy estimation.
\end{itemize}

In conclusion, while our work has advanced the understanding of min-entropy estimation through machine learning, it also highlights the practical complexity of this method and the need for more research to address its challenges.

\section{Acknowledgments}
\noindent This work is part of the R\&D project TED2021-130369B-C32, funded by MCIN/AEI/10.13039/501100011033 and by the “European Union NextGenerationEU/PRTR”, and is part of the project COMPROMISE PID2020-113795RB-C32/AEI/10.13039/501100011033. In addition, it was partially supported by project i-SHAPER PRTR-INCIBE - 2023/00623/001, which is being carried out within the framework of the Recovery, Transformation, and Resilience Plan funds, funded by the European Union (Next Generation). The authors want to thank Miguel Angel Hombrados Herrera and Gonzalo Martínez Ruiz de Arcaute for their help and fruitful comments.

\appendix 
The purpose of this Appendix is to provide the proofs for the results stated in Section \ref{theoretical_framework_section} and  to include some additional auxiliary results needed to prove them.
\renewcommand\thesubsection{\thesection.\arabic{subsection}}
\renewcommand\thesubsectiondis{\thesection.\arabic{subsection}}

\subsection{Order-p Markov Chains}
We start by revisiting the results presented in Section \ref{Order-p Markov chains}.
\begin{proof}[Proof of Lemma \ref{avg_conditional_min_entropy_bound_0}]

On the one hand, note that:
\begin{equation}\label{ineqsprevs}
 \resizebox{\hsize}{!}{
      $
\begin{aligned}
    \left\langle \max_{x_{t},\ldots,x_{t+n}} P( x_{t},\ldots,x_{t+n} \mid x_{t-1},\ldots,x_{t-p}) \right\rangle_{x_{t-1},\ldots,x_{t-p}} =\\ \sum_{x_{t-p}, \ldots, x_{t-1}} P(x_{t-p}, \ldots, x_{t-1}) \max_{x_t, \ldots , x_{t+n}}  P(x_t, \ldots , x_{t+n} \mid x_{t-1}, \ldots , x_{t-p}) \leq \\
    \sum_{x_{t-p}, \ldots, x_{t-1}} P(x_{t-p}, \ldots, x_{t-1}) \max_{x_{t-p}, \ldots , x_{t+n}}  P(x_t, \ldots , x_{t+n} \mid x_{t-1}, \ldots , x_{t-p})= \\\max_{x_{t-p}, \ldots , x_{t+n}}  P(x_t, \ldots , x_{t+n} \mid x_{t-1}, \ldots , x_{t-p}) .
\end{aligned}
$}
\end{equation}

The inequality 

 \begin{equation*}
        \resizebox{0.9\hsize}{!}{$
        \begin{aligned}
                \tilde{h}_\infty(X_t, \ldots, X_{t+n} \mid X_{t-1}, \ldots, X_{t-p}) \geq   h_\infty(X_t, \ldots, X_{t+n} \mid X_{t-1}, \ldots, X_{t-p})\ .
                   \end{aligned}      
         $}
    \end{equation*} 
follows from (\ref{ineqsprevs}) taking logarithms, dividing by $n+1$ and changing the sign.

On the other hand, the inequality
$$       
        h_{\infty}(X_t,\ldots, X_{t+n}) \geq  \tilde{h}_\infty(X_t, \ldots, X_{t+n} \mid X_{t-1}, \ldots, X_{t-p})      $$     

easily follows writing:

\begin{equation*}
 \resizebox{0.9\hsize}{!}{
      $
\begin{aligned}
    H_{\infty}(X_t,\ldots, X_{t+n}) =\\ -\log \left[\max_{x_t,\ldots,x_{t+n}} \sum_{ x_{t-1}, \ldots, x_{t-p}} P( x_{t-1}, \ldots, x_{t-p} ) P(x_t,\ldots,x_{t+n} |  x_{t-1}, \ldots, x_{t-p})\right] \ 
\end{aligned}
$}
\end{equation*}

and taking into account that:

\begin{equation*}
 \resizebox{0.9\hsize}{!}{
      $
\begin{aligned}
     \max_{x_t,\ldots,x_{t+n}} \sum_{ x_{t-1}, \ldots, x_{t-p}} P( x_{t-1}, \ldots, x_{t-p} ) P(x_t,\ldots,x_{t+n} |  x_{t-1}, \ldots, x_{t-p})   \leq \\
    \sum_{ x_{t-1}, \ldots, x_{t-p}} P( x_{t-1}, \ldots, x_{t-p} ) \max_{x_t,\ldots,x_{t+n}} P(x_t,\ldots,x_{t+n} |  x_{t-1}, \ldots, x_{t-p}) .
\end{aligned}    
$}
\end{equation*}
\end{proof}

\begin{proof}[Proof of Lemma \ref{increasing_avg_cond_min_entropy}]
   We have that:
    \begin{equation*} 
     \resizebox{\hsize}{!}{
      $
    \begin{aligned}
        P(x_t,\ldots,x_{t+n+m}\mid x_{t-1},\ldots, x_{t-p})  =\\ 
        P(x_{t+n+1},\ldots,x_{t+n+m}\mid x_{t+n},\ldots, x_{t-p})P(x_{t},\ldots,x_{t+n}\mid x_{t-1},\ldots, x_{t-p}) \leq \\ P(x_{t},\ldots,x_{t+n}\mid x_{t-1},\ldots, x_{t-p}).
    \end{aligned}
    $}
    \end{equation*}
    Since this inequality is independent of the realizations $$x_{t+n+1},\ldots, x_{t+n+m}$$ 
  it follows that:
    \begin{equation*}
    \begin{aligned}
          \max_{x_t,\ldots,x_{t+n+m}}P(x_t,\ldots,x_{t+n+m}\mid x_{t-1},\ldots, x_{t-p})\leq \\ \max_{x_t,\ldots,x_{t+n}}P(x_{t},\ldots, x_{t+n}\mid x_{t-1},\ldots, x_{t-p}).
    \end{aligned}      
    \end{equation*}
    Since logarithms are  monotonically increasing functions, we conclude that:
    \begin{equation*}
       \resizebox{\hsize}{!}{
      $
        \begin{aligned}
        \tilde{H}_{\infty}(X_t,\ldots, X_{t+n}\mid X_{t-1},\ldots, X_{t-p}) = \\ -\log \left[\sum_{x_{t-1},\ldots, x_{t-p}}P(x_{t-1},\ldots, x_{t-p})\max_{x_t,\ldots,x_{t+n}}P(x_{t},\ldots, x_{t+n}\mid x_{t-1},\ldots, x_{t-p})\right]  \leq\\
        -\log \left[\sum_{x_{t-1},\ldots, x_{t-p}}P(x_{t-1},\ldots, x_{t-p}) \max_{x_t,\ldots,x_{t+n+m}}P(x_t,\ldots,x_{t+n+m}\mid x_{t-1},\ldots, x_{t-p})\right] = \\ 
        \tilde{H}_{\infty}(X_t,\ldots, X_{t+n+m}\mid X_{t-1},\ldots, X_{t-p}).
        \end{aligned}
        $}
    \end{equation*}
\end{proof}

\begin{proof}[Proof of Theorem \ref{ConvergenceTheorem}]
   Since this result is a restatement of \cite[Theorem 8]{bozorgmanesh2016convergence}, we believe some explanation is necessary. First of all,
   what
   \cite[Theorem 8]{bozorgmanesh2016convergence} states is that for every $x,x_{t-1},\ldots,x_{t-p}\in S$:
    $$\lim_{n\to \infty}P(X_{t+n}=x\mid X_{t-1}=x_{t-1},\ldots, X_{t-p}=x_{t-p})=\pi(x)
    $$
   where $\pi$ is a stationary distribution whose existence is required as an hypothesis. 
   In our case, the existence of a stationary distribution $\pi$ follows from \cite[Theorem 7]{bozorgmanesh2016convergence} because we are assuming that $\{X_t\}_{t\in\mathbb{Z}}$ is irreducible.
   Having said that, taking the stationarity into account:
   \begin{equation*}
 \resizebox{\hsize}{!}{
      $
\begin{aligned}
P(X_t=x)=\lim_{n\to \infty} P(X_{t}=x)=\lim_{n\to \infty} P(X_{t+n}=x)=\\
\lim_{n\to \infty}\sum_{x_{t-1},\ldots, x_{t-n}}P(X_{t-1}=x_{t-1},\ldots, X_{t-p}=x_{t-p}) P(X_{t+n}=x\mid X_{t-1}=x_{t-1},\ldots, X_{t-p}=x_{t-p})=\\
\sum_{x_{t-1},\ldots, x_{t-n}}P(X_{t-1}=x_{t-1},\ldots, X_{t-p}=x_{t-p}) \lim_{n\to \infty} P(X_{t+n}=x\mid X_{t-1}=x_{t-1},\ldots, X_{t-p}=x_{t-p})=\\
\sum_{x_{t-1},\ldots, x_{t-n}}P(X_{t-1}=x_{t-1},\ldots, X_{t-p}=x_{t-p}) \pi(x)=\pi(x)
\end{aligned}
$}
\end{equation*}
and our restatement of  \cite[Theorem 8]{bozorgmanesh2016convergence} follows.  
\end{proof}

\begin{proof}[Proof of Theorem \ref{limitavgminentropy}]
On the one hand, denoting $\tau=\floor*{\sqrt{n}}$, we have that:
  \begin{equation}\label{tauineq}
        \resizebox{0.9\hsize}{!}{$
        \begin{aligned}
                   P(x_{t},\ldots, x_{t+n}\mid x_{t-1},\ldots,x_{t-p})=\frac{P(x_{t-p},\ldots, x_{t+n})}{P(x_{t-1},\ldots, x_{t-p})}\leq \\
        \frac{P(x_{t-p},\ldots,x_{t-1},x_{t+\tau},\ldots,x_{t+n})}{P(x_{t-1},\ldots, x_{t-p})}=
         P(x_{t+\tau},\ldots, x_{t+n}\mid x_{t-1},\ldots,x_{t-p}).
        \end{aligned}      
         $}
    \end{equation}  
    Therefore:
      \begin{equation*}
        \resizebox{0.9\hsize}{!}{$
        \begin{aligned}
                   \sum_{x_{t-1},\ldots,x_{t-p}}P(x_{t-1},\ldots,x_{t-p})\max_{x_{t},\ldots, x_{t+n}}  P(x_{t},\ldots, x_{t+n}\mid x_{t-1},\ldots,x_{t-p})\leq \\
                   \sum_{x_{t-1},\ldots,x_{t-p}}P(x_{t-1},\ldots,x_{t-p})\max_{x_{t+\tau},\ldots, x_{t+n}} P(x_{t+\tau},\ldots, x_{t+n}\mid x_{t-1},\ldots,x_{t-p}).
        \end{aligned}      
         $}
    \end{equation*}  
    In particular,
          \begin{equation}\label{keystep}
        \resizebox{0.9\hsize}{!}{$
        \begin{aligned}
        -\frac{1}{n+1}\log \left( \sum_{x_{t-1},\ldots,x_{t-p}}P(x_{t-1},\ldots,x_{t-p})\max_{x_{t},\ldots, x_{t+n}}  P(x_{t},\ldots, x_{t+n}\mid x_{t-1},\ldots,x_{t-p})\right)\geq \\ -\frac{1}{n+1}\log \left( \sum_{x_{t-1},\ldots,x_{t-p}}P(x_{t-1},\ldots,x_{t-p})\max_{x_{t+\tau},\ldots, x_{t+n}} P(x_{t+\tau},\ldots, x_{t+n}\mid x_{t-1},\ldots,x_{t-p})\right).
        \end{aligned}      
         $}
    \end{equation}  
    Now, the left hand side of the inequality (\ref{keystep}) is $$\tilde{h}_\infty(X_t, \ldots, X_{t+n} \mid X_{t-1}, \ldots, X_{t-p})$$
    and the limit of the right hand side of the inequality (\ref{keystep}) as $n$ goes to infinity is $h_{\infty}(\{X_t\}_{t\in\mathbb{Z}})$ because of the Convergence Theorem and the stationarity. Using Lemma \ref{avg_conditional_min_entropy_bound_0} we conclude that:
              \begin{equation*}
        \resizebox{0.9\hsize}{!}{$
        \begin{aligned}
h_{\infty}(\{X_t\}_{t\in\mathbb{Z}})\geq\lim_{n\to \infty} \tilde{h}_\infty(X_t, \ldots, X_{t+n} \mid X_{t-1}, \ldots, X_{t-p})\geq h_{\infty}(\{X_t\}_{t\in\mathbb{Z}})
        \end{aligned}      
         $}
    \end{equation*} 
    and the equality $$h_{\infty}(\{X_t\}_{t\in\mathbb{Z}})=\lim_{n\to \infty} \tilde{h}_\infty(X_t, \ldots, X_{t+n} \mid X_{t-1}, \ldots, X_{t-p})$$ 
    follows.    On the other hand:
              \begin{equation}\label{auxineq}
    \resizebox{\hsize}{!}{$
        \begin{aligned}
h_\infty(X_t, \ldots, X_{t+n} \mid X_{t-1}, \ldots, X_{t-p})=\\ -\frac{1}{n+1}\log\left[ \max_{x_{t-p},\ldots, x_{t+n}} P(x_t, \ldots, x_{t+n} \mid x_{t-1}, \ldots, x_{t-p})\right]\geq \\ -\frac{1}{n+1}\log \left[\max_{x_{t-p},\ldots, x_{t+n}} P(x_{t+\tau}, \ldots, x_{t+n} \mid x_{t-1}, \ldots, x_{t-p})\right]\underset{n\to \infty}{\longrightarrow}
 h_{\infty}(\{X_t\}_{t\in\mathbb{Z}}).
        \end{aligned}    
         $}
    \end{equation} 
    Note that  the inequality of (\ref{auxineq}) holds because of (\ref{tauineq}) and the limit of (\ref{auxineq}) holds because of the Convergence Theorem and the stationarity. Using Lemma \ref{avg_conditional_min_entropy_bound_0} we conclude that:
              \begin{equation*}
        \resizebox{\hsize}{!}{$
        \begin{aligned}
h_{\infty}(\{X_t\}_{t\in\mathbb{Z}})\geq\lim_{n\to \infty} h_\infty(X_t, \ldots, X_{t+n} \mid X_{t-1}, \ldots, X_{t-p})\geq h_{\infty}(\{X_t\}_{t\in\mathbb{Z}})
        \end{aligned}      
         $}
    \end{equation*} 
    and the result follows.
\end{proof}

\subsection{State-Independent Maximum Transition Probability and Bitflip Symmetric \texorpdfstring{Order-$p$}{Op} Markov Chains}

The following results are targeted at proving  that 
the average min-entropy of SIMTP models coincides with their min-entropy, as claimed in Proposition \ref{entropies_SIMTP}.

\begin{lemma}
\label{avg_of_conditional_n_simtp_prob}

Let $\{X_t\}_{t\in\mathbb{Z}}$ be a SIMTP order-$p$ Markov chain. Then:

\begin{equation*}
\begin{aligned}
        \left\langle \max_{x_t, \ldots, x_{t+n}} P(x_t, \ldots, x_{t+n} \mid x_{t-1}, \ldots, x_{t-p}) \right\rangle_{x_{t-1}, \ldots, x_{t-p}} = \\ \max_{
    x_{t-p}, \ldots, x_{t+n} }  P(x_t, \ldots, x_{t+n} \mid x_{t-1}, \ldots, x_{t-p}) \ .
\end{aligned}
\end{equation*}

\end{lemma}

\begin{proof}
Let us write: 
\begin{equation}
\label{prod_decomposition}
\begin{aligned}
        P(x_t, \ldots, x_{t+n} \mid x_{t-1}, \ldots, x_{t-p}) = \\ \prod_{i=0}^n  P(x_{t+i}\mid x_{t+i-1}, \ldots, x_{t+i-p}).
\end{aligned}
\end{equation}
Since $P(x_{t+i}\mid x_{t+i-1}, \ldots, x_{t+i-p})$ is independent of $x_{t+i-1}, \ldots, x_{t+i-p}$ for every $i\in\{0,\ldots, n\}$, it follows from (\ref{prod_decomposition}) that:
\begin{equation}
\label{simtp_permeability}
\begin{aligned}
        \max_{x_t, \ldots, x_{t+n}} P(x_t, \ldots, x_{t+n} \mid x_{t-1}, \ldots, x_{t-p}) = \\ \prod_{i=0}^n \max_{x_{t+i}} P(x_{t+i}\mid x_{t+i-1}, \ldots, x_{t+i-p}). 
\end{aligned}
\end{equation}
Using (\ref{simtp_permeability}), the independence of $P(x_{t+i}\mid x_{t+i-1}, \ldots, x_{t+i-p})$ with respect to $x_{t+i-1}, \ldots, x_{t+i-p}$ for every $i\in\{0,\ldots, n\}$ and the fact that the sum of probabilities of all the outcomes within a sample space is $1$, we get:
 \begin{equation*}
   \resizebox{\hsize}{!}{
      $
        \begin{aligned}
            \left\langle \max_{x_{t},\ldots,x_{t+n}} P( x_{t},\ldots,x_{t+n} \mid x_{t-1},\ldots,x_{t-p}) \right\rangle_{x_{t-1},\ldots,x_{t-p}}    
        =\\ \sum_{x_{t-1},\ldots, x_{t-p}}P(x_{t-1},\ldots, x_{t-p})\max_{x_t,\ldots,x_{t+n}}P(x_{t},\ldots, x_{t+n}\mid x_{t-1},\ldots, x_{t-p})
         =\\ \sum_{x_{t-1},\ldots, x_{t-p}}P(x_{t-1},\ldots, x_{t-p})\prod_{i=0}^n \max_{x_{t+i}} P(x_{t+i}\mid x_{t+i-1}, \ldots, x_{t+i-p})  = \\\prod_{i=0}^n \max_{x_{t+i}} P(x_{t+i}\mid x_{t+i-1}, \ldots, x_{t+i-p})
        =\\\max_{x_t, \ldots, x_{t+n}} P(x_t, \ldots, x_{t+n} \mid x_{t-1}, \ldots, x_{t-p}).
        \end{aligned}
        $}
    \end{equation*}

\end{proof}

\begin{proposition}
\label{simtp_chain_rule}
Any order-$p$ SIMTP satisfies the following decomposition:
\begin{equation*}
 \resizebox{\hsize}{!}{
      $
\begin{aligned}
     H_{\infty}(X_t,\ldots, X_{t+n}) = \\ H_{\infty}(X_t,\ldots, X_{t+p-1}) + H_{\infty}(X_{t+p}, \ldots, X_{t + n } \mid X_{t + p - 1}, \ldots, X_{t}) 
\end{aligned}
$}
\end{equation*}

\end{proposition}
\begin{proof}

Let us write:

\begin{equation*}
   \resizebox{\hsize}{!}{
      $
\begin{aligned}
        H_{\infty}(X_t,\ldots, X_{t+n})   =-\log \left[\max_{x_t,\ldots,x_{t+n}}P(x_t,\ldots,x_n)\right] = \\ - \log \left[\max_{x_t,\ldots,x_{t+n}} P(x_t, \ldots, x_{t + p -1})  P(x_{t+p}, \ldots, x_{t + n } \mid x_{t + p - 1}, \ldots, x_{t}) \right].
\end{aligned}
$}
\end{equation*}

If the process is SIMTP then we can reach the maximum over $P(x_{t+p}, \ldots, x_{t + n })$ independently of the values required to maximize over $P(x_{t}, \ldots, x_{t + p -1 })$. In other words, we can maximize both  probabilities independently:

\begin{equation*}
\resizebox{\hsize}{!}{
      $
\begin{aligned}
     H_{\infty}(X_t,\ldots, X_{t+n})  = \\ - \log \left[\max_{x_t,\ldots,x_{t+p-1}}   P(x_{t}, \ldots, x_{t + p -1} )  \right] 
     - \log \left[\max_{x_t,\ldots,x_{t+n}} P(x_{t+p}, \ldots, x_{t + n } \mid x_{t + p - 1}, \ldots, x_{t}) \right]   = \\
     H_{\infty}(X_t,\ldots, X_{t+p-1}) + H_{\infty}(X_{t+p}, \ldots, X_{t + n } \mid X_{t + p - 1}, \ldots, X_{t}).
\end{aligned}
$}
    \end{equation*}
\end{proof}

\begin{proof}[Proof of Proposition \ref{entropies_SIMTP}]
On the one hand, by Proposition \ref{simtp_chain_rule} we have:

\begin{equation*}
 \resizebox{0.95\hsize}{!}{
      $
\begin{aligned}
        H_{\infty}(X_t,\ldots, X_{t+n})   = \\ H_{\infty}(X_t,\ldots, X_{t+p-1}) + H_{\infty}(X_{t+p}, \ldots, X_{t + n } \mid X_{t + p - 1}, \ldots, X_{t})   = \\
     H_{\infty}(X_t,\ldots, X_{t+p-1}) - \log \left[ \max_{x_t} P(x_t\mid x_{t-1}, \ldots , x_{t-p})^{n-p}\right].
\end{aligned}
$}
\end{equation*}

Then, for finite $p$, as the first term is bounded
\begin{equation}
\label{minentropySIMTP}
\begin{aligned}
       h_{\infty}(\{X_t\}_{t \in \mathbb{Z}})= \lim_{n\to \infty} \frac{1}{n+1} H_{\infty}(X_t,\ldots, X_{t+n}) = \\  - \log \left[ \max_{x_t} P(x_t\mid x_{t-1}, \ldots , x_{t-p})\right].
\end{aligned}
\end{equation}

On the other hand:

\begin{equation}
\label{avgcondminentropySIMTP}
\resizebox{\hsize}{!}{
      $
\begin{aligned}
    \tilde{h}_\infty(X_t, \ldots, X_{t+n} \mid X_{t-1}, \ldots, X_{t-p})  = \\ -\log \left[\left(\sum_{x_{t-1},\ldots, x_{t-p}}P(x_{t-1},\ldots, x_{t-p})\max_{x_t,\ldots,x_{t+n}}P(x_{t},\ldots, x_{t+n}\mid x_{t-1},\ldots, x_{t-p})\right)^{\frac{1}{n+1}}\right] = \\  
     -\log \left[\left(\sum_{x_{t-1},\ldots, x_{t-p}}P(x_{t-1},\ldots, x_{t-p})\max_{x_t}P(x_{t}\mid x_{t-1},\ldots, x_{t-p})^n\right)^{\frac{1}{n+1}}\right] = \\  
     -\log \left[\left(\sum_{x_{t-1},\ldots, x_{t-p}}P(x_{t-1},\ldots, x_{t-p})\right)^{\frac{1}{n+1}}\max_{x_t} P(x_{t}\mid x_{t-1},\ldots, x_{t-p})\right] = \\  
     - \log \left[ \max_{x_t} P(x_t\mid x_{t-1}, \ldots , x_{t-p})\right].
    \end{aligned}
    $}
\end{equation}

The result follows putting (\ref{minentropySIMTP}) and (\ref{avgcondminentropySIMTP}) together.
    
\end{proof}

We end this subsection of the Appendix by proving that Bitflip-Symmetric
Markov Chains with lag-p point-to-point correlations are SIMTP, as stated in Lemma \ref{bitflip_symmetric_simtp}.

\begin{proof}[Proof of Lemma \ref{bitflip_symmetric_simtp}]
Let $x_{t-p}\in\{0,1\}$. Then, by the definition of bitflip symmetry:
\begin{equation*}
 \resizebox{\hsize}{!}{
      $
\begin{aligned}
     \max_{x_t} P(x_t \mid x_{t-p}) = \max \{ P(X_t=0 \mid x_{t-p}), P(X_t=1 \mid  x_{t-p})\} =\\ \max \{ P(X_t=1 \mid 1\oplus x_{t-p}), P(X_t=0 \mid 1\oplus  x_{t-p})\}
     =\\ \max_{x_t} P(x_t \mid 1\oplus x_{t-p})
\end{aligned}   
$}
\end{equation*}
and the result follows.
\end{proof}

\subsection{Generalized Binary Autoregressive Models}

The last subsection of the Appendix is devoted to gather some interesting properties of gbAR(p) models. In particular, they will allow to prove the formula for calculating the min-entropy of uniform noise and positive gbAR(p) models stated in Proposition \ref{min_entropy_gbarp}.

 \begin{remark} \label{uniform_noise_probs}
     Let $\{X_t\}_{t\in \mathbb{Z}}$ be a uniform noise gbAR(p) model. Then
     $P(X_t=0)=P(X_t=1)=\frac{1}{2}$ for every $t\in\mathbb{Z}$ by \cite[Lemma 1]{jentsch2019generalized}. 
 \end{remark}

\begin{remark}
\label{transition_probs_gbAR}
    By \cite[Lemma 1]{jentsch2019generalized}  the  transition probabilities of a gbAR(p) model $\{X_t\}_{t\in\mathbb{Z}}$ can be written as:
\begin{equation*}
 \label{gbar_p_transition_probability}
 \resizebox{\hsize}{!}{
      $
\begin{aligned}
     P(x_{t} \mid x_{t-1}, \ldots, x_{t-p}) = \\ \sum_{i=1}^{p} [\mathds{1}_{\lbrace \alpha_i \geq 0 \rbrace} |\alpha_i| \cdot (1 \oplus x_{t} \oplus x_{t-i})  + \mathds{1}_{\lbrace \alpha_i < 0 \rbrace} |\alpha_i| \cdot (x_{t} \oplus x_{t-i})] +\beta\cdot P(e_t=x_t) \ .
\end{aligned}
$}
\end{equation*}

If $\{X_t\}_{t\in\mathbb{Z}}$ is a uniform noise and positive gbAR(p) model this writes as

\begin{equation*}
\label{gbar_p_u+_transition_probability}
\begin{aligned}
     P(x_{t} \mid x_{t-1}, \ldots, x_{t-p}) =  \sum_{i=1}^{p} \alpha_i \cdot (1 \oplus x_{t} \oplus x_{t-i})   +\frac{\beta}{2}
\end{aligned}
\end{equation*}

and this conditional probability reaches the maximum value
\begin{equation*}
    \max_{x_{t-1}, \ldots, x_{t-p}}  P(x_{t} \mid x_{t-1}, \ldots, x_{t-p}) = \sum_{i=1}^{p} \alpha_i + \frac{\beta}{2} = 1 - \frac{\beta}{2} \ 
\end{equation*}
when $x_t=x_{t-1}=\cdots=x_{t-p}$. In particular, a realization $x_{t-p},\ldots, x_{t+n}$  of $X_{t-p},\ldots, X_{t+n}$ such that $$x_{t-p}=\cdots= x_{t+n}$$ maximizes the conditional probability $P(x_{k}\mid x_{k-1},\ldots, x_{k-p})$ for every $k\in\{t,\ldots, t+n\}$ and it follows that:
\begin{equation*}
\begin{aligned}
        \max_{x_{t-p},\ldots, x_{t+n}}P(x_{t},\ldots, x_{t+n}\mid x_{t-1},\ldots, x_{t-p})=\\ \max_{x_{t-p},\ldots, x_{t+n}}\prod_{k=t}^{t+n} P(x_{k}\mid x_{k-1},\ldots, x_{k-p})= \\ \prod_{k=t}^{t+n} \max_{x_{k},\ldots, x_{k-p}}P(x_{k}\mid x_{k-1},\ldots, x_{k-p})=
        \left(1 - \frac{\beta}{2}\right)^{n+1}.
\end{aligned}
\end{equation*}

\end{remark}

\begin{remark}
   Note that a gbAR(p) model has point-to-point lag-p correlations if  $a_t^{(i)}=\delta_{i,p}$ for every $i\in\{1,\ldots,p\}$ and every $t\in \mathbb{Z}$. This is achieved when $\alpha_i=0$ for every $i\in\{1,\ldots, p-1\}$.
\end{remark}

\begin{lemma}\label{gbarp_lag_bitflip}
    A uniform noise gbAR(p) model with point-to-point lag-p correlations is bitflip-symmetric.
\end{lemma}
\begin{proof}
     By the definition of conditional probability and the definition of gbAR(p) model with point-to-point lag-p correlations:
    \begin{equation}
    \label{factorization} P(x_t,\ldots, x_{t+n}\mid x_{t-1},\ldots,x_{t-p})=\prod_{i=0}^{n} P(x_{t+i}\mid x_{t+i-p}).
    \end{equation}
  Then, by Remark \ref{uniform_noise_probs} and Remark \ref{transition_probs_gbAR}:
    \begin{equation}
        \label{bitflipequality} 
        \begin{aligned}
            P(x_{t+i}\mid x_{t+i-p})  = \\ [\mathds{1}_{\lbrace \alpha_p \geq 0 \rbrace} |\alpha_p| \cdot (1 \oplus x_{t+i} \oplus x_{t+i-p})  + \\ \mathds{1}_{\lbrace \alpha_p < 0 \rbrace} |\alpha_p| \cdot (x_{t+i} \oplus x_{t+i-p})] + \frac{\beta}{2}= \\  
           [\mathds{1}_{\lbrace \alpha_p \geq 0 \rbrace} |\alpha_p| \cdot (1 \oplus 1 \oplus x_{t+i} \oplus 1 \oplus x_{t+i-p})  + \\ \mathds{1}_{\lbrace \alpha_p < 0 \rbrace} |\alpha_p| \cdot (1 \oplus x_{t+i} \oplus 1 \oplus x_{t+i-p})] +   \frac{\beta}{2}
           = \\  
           P(1\oplus x_{t+i}\mid 1\oplus x_{t+i-p}).
        \end{aligned}
    \end{equation}
    Putting (\ref{factorization}) and (\ref{bitflipequality}) together:
     \begin{equation*}
     \begin{aligned}
          P(x_t,\ldots, x_{t+n}\mid x_{t-1},\ldots,x_{t-p})= \\ \prod_{i=0}^{n} P(x_{t+i}\mid x_{t+i-p})=
          \prod_{i=0}^{n} P(1\oplus x_{t+i}\mid 1\oplus x_{t+i-p})  =\\ P(1\oplus x_t,\ldots, 1\oplus x_{t+n}\mid 1\oplus x_{t-1},\ldots,1\oplus x_{t-p}).
     \end{aligned}   
    \end{equation*}
\end{proof}

\begin{proposition}\label{gbarpu+convergence}
    Let $\{X_t\}_{t\in\mathbb{Z}}$ be a uniform noise and positive gbAR(p) model. Then $\{X_t\}_{t\in\mathbb{Z}}$ satisfies the hypothesis of the Convergence Theorem, i.e. $\{X_t\}_{t\in\mathbb{Z}}$ is an irreducible and aperiodic stationary order-$p$ Markov
chain with finite state-space. 
\end{proposition}

\begin{proof}
    $\{X_t\}_{t\in\mathbb{Z}}$ is a stationary order-$p$ Markov
chain with finite state-space $S=\{0,1\}$ by Definition \ref{definition_gbarp}. Now,  $\{X_t\}_{t\in\mathbb{Z}}$ is irreducible because  for every $x_{t-p},\ldots, x_t\in S$ we have:
\begin{equation}\label{irreducibilitygbarpu+}
\begin{aligned}
     P(x_{t} \mid x_{t-1}, \ldots, x_{t-p}) =  \sum_{i=1}^{p} \alpha_i \cdot (1 \oplus x_{t} \oplus x_{t-i}) +\frac{\beta}{2}>0 .
\end{aligned}
\end{equation}
Finally, since $$\sum_{x_{t-2},\ldots, x_{t-p}}P(x_{t-2},\ldots, x_{t-p})=1,$$ there exist $y_{t-2},\ldots, y_{t-p}\in S$
 such that  $P(y_{t-2},\ldots, y_{t-p})>0$. Then, by equation (\ref{irreducibilitygbarpu+}):
\begin{equation*}
 \resizebox{\hsize}{!}{
      $
\begin{aligned}
     P(x_{t} \mid x_{t-1}) =\sum_{x_{t-2},\ldots, x_{t-p}}P(x_{t-2},\ldots, x_{t-p})P(x_t\mid x_{t-1},x_{t-2},\ldots, x_{t-p})\geq\\ P(y_{t-2},\ldots, y_{t-p})P(x_t\mid x_{t-1},y_{t-2},\ldots, y_{t-p})>0
\end{aligned}
$}
\end{equation*}
and we conclude that $\{X_t\}_{t\in\mathbb{Z}}$ is aperiodic.
\end{proof}

\begin{proof}[Proof of Proposition \ref{min_entropy_gbarp}]
   Since $\{X_t\}_{t\in\mathbb{Z}}$ satisfies the hypothesis of the Convergence Theorem \ref{ConvergenceTheorem} by Proposition \ref{gbarpu+convergence}, the equalities
           \begin{equation*}
        \resizebox{\hsize}{!}{$
        \begin{aligned}
    h_{\infty}(\{X_t\}_{t\in\mathbb{Z}})=\lim_{n\to \infty} \tilde{h}_\infty(X_t, \ldots, X_{t+n} \mid X_{t-1}, \ldots, X_{t-p})=\\
    \lim_{n\to \infty} h_\infty(X_t, \ldots, X_{t+n} \mid X_{t-1}, \ldots, X_{t-p})
        \end{aligned}     
         $}
    \end{equation*}   
    follow from Theorem \ref{limitavgminentropy}. Now, 
       \begin{equation*}
        \resizebox{\hsize}{!}{$
        \begin{aligned}
       h_{\infty}(X_t,\ldots, X_{t+n}\mid X_{t-1},\ldots, X_{t-p})=\\ -\frac{1}{n+1}\log \left[\max_{x_{t-p},\ldots,x_{t+n}}P(x_{t},\ldots, x_{t+n}\mid x_{t-1},\ldots, x_{t-p})\right]=\\
       -\frac{1}{n+1}\log \left[\max_{x_{t-p},\ldots,x_{t+n}}\prod_{i=0}^{n}P(x_{t+i}\mid x_{t+i-1},\ldots, x_{t+i-p})\right]=\\
        -\frac{1}{n+1}\log \left[\prod_{i=0}^{n}\left(\sum_{i=1}^{p} \alpha_i  +\frac{\beta}{2}\right)\right]=-\frac{1}{n+1}\log \left[\left(\sum_{i=1}^{p} \alpha_i  +\frac{\beta}{2}\right)^{n+1}\right]=\\
        -\frac{1}{n+1}\log \left[\left(1-\frac{\beta}{2}\right)^{n+1}\right]=
        -\log \left(1-\frac{\beta}{2}\right).
        \end{aligned}     
         $}
    \end{equation*} 
\end{proof}

 \printbibliography

\newpage

\section{Biography Section}

\begin{IEEEbiography}[{\includegraphics[width=1in,height=1.25in,clip,keepaspectratio]{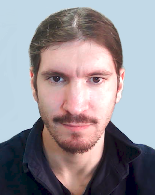}}]{Francisco Javier Blanco Romero}
received his B.Sc. in Physics from the Universidad Complutense de Madrid and is finalizing his M.Sc. degree in Robotics at Universidad Miguel Hernández de Elche. He has worked in the private sector as a software engineer and researcher, focusing on Internet of Things (IoT) projects in the environmental and mobility domains. Currently, he is a Research Support Technician at Universidad Carlos III de Madrid, involved in the Quantum-based Resistant Architectures and Techniques project.
\end{IEEEbiography}

\begin{IEEEbiography}
[{\includegraphics[width=1in,height=1.25in,clip,keepaspectratio]{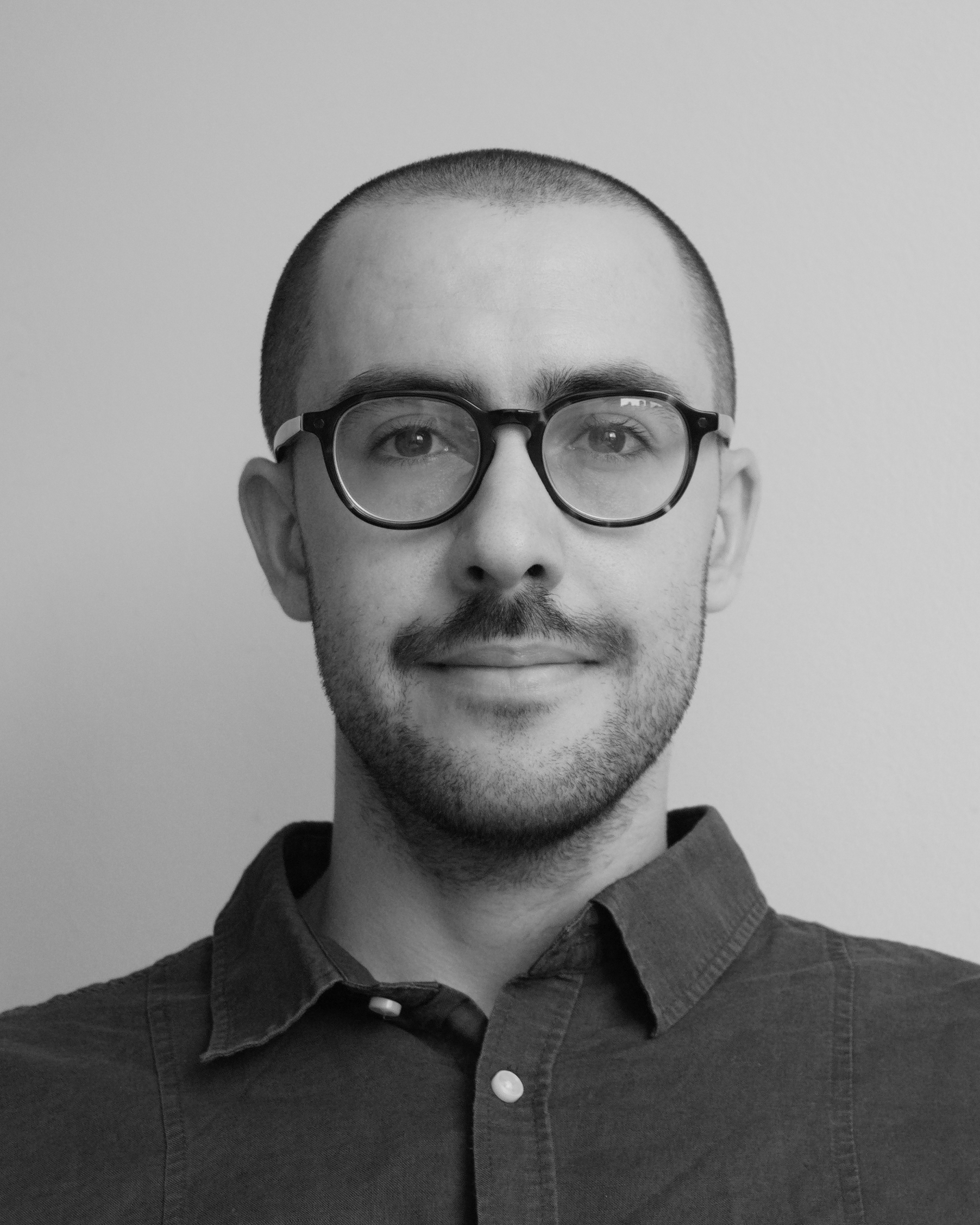}}]{Vicente Lorenzo}
received the B.Sc. degree in mathematics from the Universidad Complutense de Madrid, the M.Sc. degree in mathematics and applications from the Universidad Autónoma de Madrid and the Ph.D. degree in mathematics from Instituto Superior Técnico - Universidade de Lisboa. He was a Lecturer with Instituto Superior de Economia e Gestão - Universidade de Lisboa, Instituto Superior de Ciências do Trabalho e da Empresa - Instituto Universitário de Lisboa and Universidad CEU San Pablo. He is currently a Project Researcher with Universidad Carlos III de Madrid where he is involved in the Quantum-based Resistant Architectures and Techniques project.
\end{IEEEbiography}

\begin{IEEEbiography}
[{\includegraphics[width=1in,height=1.25in,clip,keepaspectratio]{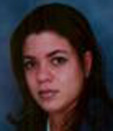}}]{Florina Almenares Mendoza}
received the M.Sc. degree in Telematic and the Ph.D. degree from the University Carlos III of Madrid, in 2003 and 2006, respectively. Since 2008, she has worked as an Associate Professor with the Department of Telematic Engineering, University Carlos III of Madrid. Her research interests include post-quantum cryptography, cybersecurity, machine learning, trust and reputation management models, identity management, secure architectures, and risk assessment. This research has been recently applied to ubiquitous computing and IoT, smart grids, and smart cities. She received the IEEE Chester Sall Award (2012).
\end{IEEEbiography}

\begin{IEEEbiography}
[{\includegraphics[width=1in,height=1.25in,clip,keepaspectratio]{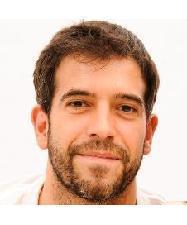}}]{Daniel Díaz Sánchez}
received the B. Eng. degree in Telecommunications in 2003, M.Sc. in Telematics in 2007, and Ph.D. in 2009 from Carlos III University of Madrid. From 2004 to 2006, he was a Researcher and became a Teaching Assistant in 2005. Since 2010, he has been an associate professor at the University Carlos III of Madrid. His research interests include distributed authentication/authorization, cybersecurity, distributed and fog computing, IoT, and related concepts. Dr. Díaz-Sánchez has received several awards and honors, including the Especial Ph.D. award from University Carlos III of Madrid (2009), the best Ph.D. thesis award on electronic commerce by the National Telecommunication Engineering Association (2009), and the IEEE Chester Sall Award (2012).
\end{IEEEbiography}

\vfill

\end{document}